\documentclass{article}

\usepackage[main, final]{neurips_2026}

\usepackage[utf8]{inputenc}
\usepackage[T1]{fontenc}
\usepackage[dvipsnames]{xcolor}
\definecolor{oxfordblue}{RGB}{0,33,71}
\usepackage[colorlinks=true, allcolors=oxfordblue]{hyperref}
\usepackage{url}
\usepackage{booktabs}
\usepackage{amsmath,amsfonts,amssymb}
\usepackage{nicefrac}
\usepackage{microtype}
\usepackage{xcolor}
\usepackage{graphicx}
\usepackage{amsthm}
\usepackage{subcaption}
\usepackage{makecell}
\usepackage{array}
\usepackage{multirow}

\theoremstyle{definition}
\newtheorem{definition}{Definition}

\theoremstyle{plain}
\newtheorem{lemma}[definition]{Lemma}
\newtheorem{corollary}[definition]{Corollary}
\newtheorem{proposition}[definition]{Proposition}

\theoremstyle{remark}

\title{Faithful Embeddings of Irregular and Asynchronous Data for Online Log-NCDEs}
\newif\ifthreeauthorrows
\threeauthorrowstrue     

\newif\ifthreeauthorrows
\threeauthorrowsfalse  

\author{
    Benjamin Walker \\
    University of Oxford \\
    Continuous Intelligence Limited
    \And
    Alexandre Bloch \\
    University of Oxford
    \ifthreeauthorrows \AND \else \And \fi
    Lingyi Yang \\
    University of Nottingham \\
    Continuous Intelligence Limited
    \ifthreeauthorrows \And \else \AND \fi
    Sam Morley \\
    University of Oxford
    \ifthreeauthorrows \AND \else \And \fi
    Terry Lyons \\
    University of Oxford \\
    Imperial College London
}

\begin{document}

\maketitle

\begin{abstract}
Continuous-time models are a natural choice for irregular and asynchronous data.
A central design choice is how to embed discrete observations into continuous time. 
Interpolation- and imputation-based embeddings reconstruct a continuous observation path, making the model sensitive to the choice of reconstruction. 
We show that this reconstruction step is unnecessary; under mild conditions, compact-set universality on the model input space transfers to the data space whenever the embedding from data to input is continuous and injective.
Guided by this result, and building on the rectilinear control path for Neural Controlled Differential Equations (NCDEs), we introduce a continuous and injective embedding for Log-NCDEs, a universal class of continuous-time models. 
Our approach records observations as increments and composes them over arbitrary query intervals to directly form log-signatures.
This provides interval-level summaries without first interpolating the observed variables, while supporting online computation.
Experiments on synthetic controlled dynamics and real-world time-series datasets show that the representation is accurate, efficient, and robust to irregular, asynchronous, and sparse observations.
\end{abstract}

\section{Introduction}

\begin{figure}[t]
    \centering
    \includegraphics[width=1.0\textwidth]{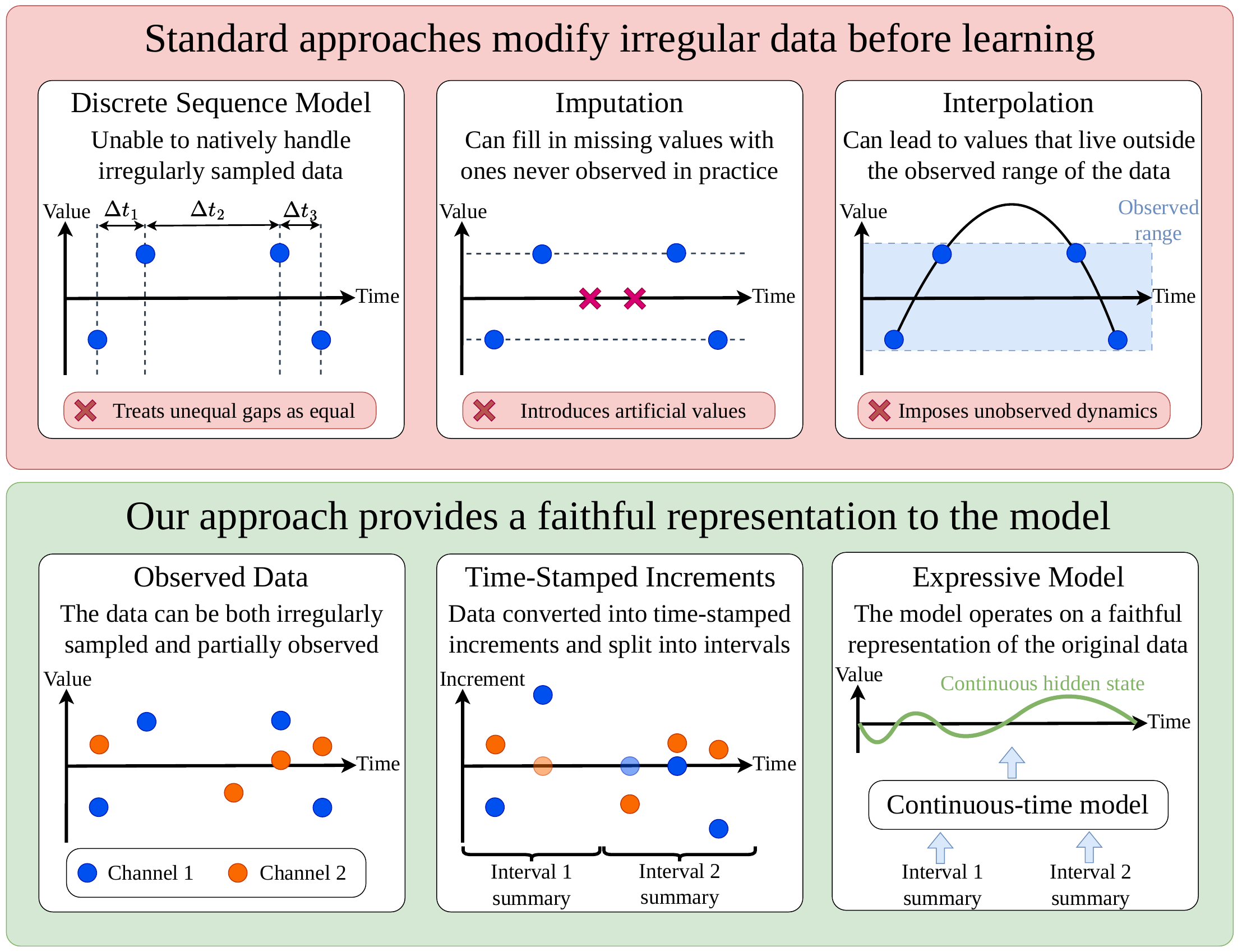}
    \caption{\textbf{Faithful representations of irregular data.}
    Imputation and interpolation complete the data before modelling, introducing values or dynamics that may be unrealistic. 
    Our approach records the observed stream as time-stamped increments and computes interval summaries of a faithful (continuous and injective) representation of the discrete and irregular data.}
    \label{fig:whynotinterpolate}
\end{figure}

Sequence architectures, such as Transformers \citep{vaswani2017attention}, Structured State-Space Models (SSMs) \citep{gu2021efficiently, gu2024mamba}, and Recurrent Neural Networks (RNNs), are built around discrete updates on a prescribed grid \citep{ elman1990finding, hochreiter1997lstm}. 
This discretisation is natural for text and regularly sampled time series, but is limiting in interactive and physical settings, where observations arrive irregularly, actions are taken asynchronously, and the underlying environment continues to evolve between measurements \citep{marvasti2001nonuniform, heemels2012introduction}. 
Continuous-time machine learning provides a natural framework for such data \citep{pearlmutter1989learning, martinez1992discrete}.
Given a continuous-time embedding of the observations, the hidden state can evolve continuously and be queried at arbitrary times \citep{kidger2020neuralcde}. 

A central design choice is how to construct this embedding. 
One common strategy is to complete the data by imputation, interpolation, or learned interpolation before applying a model \citep{che2018recurrent, shukla2019interpolation, shukla2021multi, challu2023nhits, wang2025deep}. 
Neural Controlled Differential Equations (NCDEs) and ContiFormer are prominent examples of this approach \citep{kidger2020neuralcde, morrill2022interpolation, jhin2022exit, chen2023contiformer}.
Other models avoid explicit interpolation by combining continuous latent dynamics with discrete updates, as in ODE-RNNs and GRU-ODE-Bayes \citep{ODERNN, GRU-ODE}.
These approaches differ in how they represent discrete, irregular, and asynchronous data in continuous time.

The central principle of this paper is that a continuous-time representation of discrete observations should preserve only the information that was observed, rather than guess the unobserved behaviour between measurements.
The latter should form part of the modelling as opposed to part of the representation.
When guesses are built into the representation, the model may consume unrealistic inputs, as illustrated in Figure~\ref{fig:whynotinterpolate}.
\citet{morrill2022interpolation} studied control-path embeddings for NCDEs, exploring properties such as online, smooth, bounded, and unique, or equivalently injective.
For Log-NCDEs, smoothness is no longer intrinsic, since the model uses log-signatures on intervals rather than pointwise values of a control path \citep{Walker2024LogNCDE}.
We prove that, when the data and model input spaces are Hausdorff, universality of the model on compact subsets of its input space transfers to the data space as long as the data embedding is continuous and injective.
We call such embeddings \textit{faithful}.
Since NCDEs and Log-NCDEs are universal on path space \citep{kidger2020neuralcde, Walker2024LogNCDE}, this result applies to both model classes for faithful observation embeddings.

This leads to three desiderata for a Log-NCDE control-path embedding. 
It should be faithful, computable online, and support efficient calculation of the embedded path’s log-signature over arbitrary query intervals.
We construct such an embedding by mapping each observation to a local log-signature contribution that records which channels were observed and their values.
These contributions compose over arbitrary query intervals, so the representation can be built directly from the observations in each interval.
Continuously observed variables, such as time, contribute between observations.
Figure~\ref{fig:whynotinterpolate} illustrates our approach.

This yields a faithful representation of the observed stream without interpolation, decouples the input sampling times from the output grid, and supports parallel-in-time computation.
The resulting algebraic formulation can also accommodate observations beyond first-order channel values, such as the signed area.
In the first-order case with time as the only continuously observed variable, the continuous path corresponding to our embedding is the rectilinear control path introduced for NCDEs by \citet{morrill2022interpolation}. 
Our construction can therefore be viewed as a Log-NCDE-compatible generalisation of the rectilinear scheme.


Our contributions are as follows:
\begin{itemize}
    \item We formalise the notion of a faithful embedding and prove that faithful embeddings transfer universality from compact subsets of one Hausdorff space to another in Section~\ref{sec:theory}.
    \item We introduce a faithful control-path embedding for Log-NCDEs whose interval log-signatures can be computed efficiently over arbitrary intervals in Section~\ref{sec:procedure}.
    \item We demonstrate the computational benefits and robustness to irregularity of the resulting models on synthetic controlled systems and real-world time-series classification in Section~\ref{sec:experiments}.
\end{itemize}

\section{Faithful Embeddings of Irregular Time Series for Log-NCDEs}
\label{sec:method}

\subsection{Faithful Embeddings}
\label{sec:theory}

Before modelling, data must be embedded into the model input space.
We call such an embedding \textit{faithful} if it is continuous and injective, thus ensuring that nearby data remain nearby after embedding and that distinct data remain distinguishable.
Together, these conditions are sufficient for universality to transfer from the downstream model to the original data space, as formalised by the following lemma.

\begin{lemma}[Universality Transfer]
\label{lem:transfer}
Let $\mathcal{X}$ and $\mathcal{P}$ be Hausdorff topological spaces, and let $\varphi : \mathcal{X} \to \mathcal{P}$ be a continuous injection.
Suppose $\mathcal{F} \subseteq C(\mathcal{P}; \mathbb{R})$ is a family of continuous real-valued functions such that for every compact $K \subseteq \mathcal{P}$, the restrictions $\mathcal{F}|_K$ are dense in $C(K; \mathbb{R})$ in the uniform topology.
Then for every compact $K_0 \subseteq \mathcal{X}$, the restrictions $(\mathcal{F} \circ \varphi)|_{K_0}$ are dense in $C(K_0; \mathbb{R})$ in the uniform topology, where $\mathcal{F} \circ \varphi := \{ f \circ \varphi : f \in \mathcal{F} \}$.
\end{lemma}




The proof is in Appendix \ref{app:faithful-embeddings}. With $\mathcal{X}$ the data space, $\mathcal{P}$ the model input space, and $\varphi$ the embedding, Lemma \ref{lem:transfer} shows that faithful embeddings transfer universality back to the original data. This recovers the usual signature-universality argument: Stone--Weierstrass gives universality on signature space, and injectivity transfers it back to paths \citep{StoneWeierstrass, levin2016learningpastpredictingstatistics}. Appendix \ref{app:faithful-embeddings} also notes the obvious converse obstruction: if $\varphi$ identifies two data objects, no downstream model can separate them.
Appendix~\ref{app:limitations} discusses the scope of this universality statement and its distinction from sample efficiency or optimisation guarantees.

We now construct a faithful embedding of an irregular observation stream as a path.
The construction is tailored to Log-NCDEs, which require interval log-signatures rather than pointwise values of an interpolated path.
It therefore computes the required log-signatures directly from the observations, without interpolating the data.
The procedure follows the signature-computation viewpoint of \citet{morley2024roughpy}.
We begin with an introduction to the tensor algebra, signatures, and log-signatures.

\subsection{Tensor algebra, signatures, and log-signatures}

Let $X:[0,T]\to\mathbb{R}^{d_X}$ be a path of bounded variation, with coordinates $X=(X^1,\ldots,X^{d_X})$.
The signature of $X$ over $[s,t]$ records all iterated integrals of its coordinate increments,
\begin{equation}
    S_{s,t}(X)^I
    =
    \int_{s<u_1<\cdots<u_k<t}
    \mathrm{d}X^{i_1}_{u_1}\cdots \mathrm{d}X^{i_k}_{u_k},
\end{equation}
where $I=i_1\cdots i_k$ is a word of length $|I|=k$ with letters in $\{1,\ldots,d_X\}$  \citep{lyons2007differential, friz2010multidimensional}.
The empty word is denoted by $\emptyset$, with $S_{s,t}(X)^\emptyset=1$.
The $k^{\text{th}}$ level consists of all signature coordinates $S_{s,t}(X)^I$ indexed by words $I$ of length $k$.
The collection of all levels forms an element of the tensor algebra $T((\mathbb{R}^{d_X}))$, equipped with the product
\begin{equation}
    (\mathbf{a}\otimes\mathbf{b})^I
    =
    \sum_{I=JK}\mathbf{a}^J\mathbf{b}^K,
\end{equation}
where the sum is over all decompositions of $I$ into two words $J$ and $K$.
The unit is $\mathbf{1}$, with $\mathbf{1}^{\emptyset}=1$ and $\mathbf{1}^I=0$ for $|I|\geq1$.
Chen's identity \citep{Chen1954Iterated} states that, for $s\leq u\leq t$,
\begin{equation}
    S_{s,t}(X)
    =
    S_{s,u}(X)\otimes S_{u,t}(X),
\end{equation}
so this product is the algebraic counterpart of concatenating intervals \citep{SigPrimer,lyons2025signaturemethodsmachinelearning}. 

The tensor product is non-commutative, giving the Lie bracket $[\mathbf{a},\mathbf{b}] = \mathbf{a}\otimes\mathbf{b}-\mathbf{b}\otimes\mathbf{a}$. The free Lie algebra $\mathfrak{L}((\mathbb{R}^{d_X}))$ is generated by degree-one tensors under this bracket \citep{reutenauer1993free}.
The log-signature is defined
\begin{equation}
    \log S_{s,t}(X)
    =
    \sum_{m=1}^{\infty}
    \frac{(-1)^{m+1}}{m}
    \left(S_{s,t}(X)-\mathbf{1}\right)^{\otimes m}
\end{equation}
and lives in $\mathfrak{L}((\mathbb{R}^{d_X}))$ as signatures of bounded variation paths are group-like elements of $T((\mathbb{R}^{d_x}))$ \citep{Chen1957IntegrationOP}. 
Although it contains the same information as the signature, each level is represented in the lower-dimensional free Lie algebra rather than in the full tensor algebra.
The depth-$N$ truncated tensor algebra $T^N(\mathbb{R}^{d_X})$ keeps only levels with degree at most $N$, and $S^N_{s,t}(X)$ and $\log S^N_{s,t}(X)$ are the projections of $S_{s,t}(X)$ and $\log S_{s,t}(X)$ onto these levels.
We now introduce Log-NCDEs, which interact with the driving path on an interval through the truncated log-signature.

\subsection{Neural controlled differential equations}

An NCDE models a continuous-time hidden state $h_t\in\mathbb{R}^{d_h}$ driven by a path $X$ through a parametrised vector field $f_\theta:\mathbb{R}^{d_h}\to\mathbb{R}^{d_h\times d_X}$ \citep{kidger2020neuralcde},
\begin{equation}
    \mathrm{d}h_t = f_\theta(h_t)\,\mathrm{d}X_t.
\end{equation}
For $a\in\mathbb{R}^{d_X}$, $f_\theta(\cdot)a:\mathbb{R}^{d_h}\to\mathbb{R}^{d_h}$ denotes the vector field obtained by multiplying the matrix output of $f_\theta$ by $a$.
The vector field may be nonlinear, or may act linearly on the hidden state as in a Linear NCDE, $f_\theta(h_t)\mathrm{d}X_t=\sum_{i=1}^{d_X} A^i_\theta h_t\mathrm{d}X^i_t$ \citep{cirone2024deepSSM}. 
Predictions are produced from the hidden state by a readout map $l_\psi$. 
With a linear readout of the terminal hidden state, both nonlinear and Linear NCDEs are universal for continuous path-to-point maps \citep{kidger2020neuralcde}.
With a nonlinear readout of the hidden path, they are universal for continuous online path-to-path maps \citep{cirone2024deepSSM}.
Structured Linear CDEs (SLiCEs) restrict the matrices $A^i_\theta$ to structured matrix families for efficiency.
Several structures, including block-diagonal, diagonal-plus-low-rank, sparse, and the exponentially weighted signature retain the same path-to-point and path-to-path universality guarantees as dense Linear NCDEs \citep{walker2025structuredlinearcdesmaximally, movahedi2025fixedpointrnnsdiagonaldense, bloch2026extending}.

The Log-ODE method approximates CDE dynamics on an interval by constructing an autonomous ODE from the truncated log-signature of the driving path and the iterated Lie brackets of the vector fields \citep{CASTELL199513}.
A Log-NCDE applies this approximation to NCDEs, so that the hidden-state evolution depends on the driving path only through its log-signature over intervals.
Let $\bar f_\theta$ be the Lie algebra homomorphism extending $f_\theta$ from degree-one tensors to $\mathfrak{L}^N(\mathbb{R}^{d_X})$, so that $\bar f_\theta(\cdot)a = f_\theta(\cdot)a$ for $a\in\mathbb{R}^{d_X}$ and recursively
\begin{equation}
    \bar f_\theta(\cdot)[a,b]
    =
    [\bar f_\theta(\cdot)a,\bar f_\theta(\cdot)b]
    =
    J_{\bar f_\theta(\cdot)b}\bar f_\theta(\cdot)a
    -
    J_{\bar f_\theta(\cdot)a}\bar f_\theta(\cdot)b,
\end{equation}
where $J_{\bar f_\theta(\cdot)b}$ denotes the Jacobian of the vector field $\bar f_\theta(\cdot)b:\mathbb{R}^{d_h}\to\mathbb{R}^{d_h}$ and $f_\theta$ needs to be $\mathrm{Lip}(\gamma)$ for $\gamma \geq N$ \citep{Walker2024LogNCDE}.
On an interval $[r_i,r_{i+1}]$, the Log-ODE update solves 
\begin{equation}
    \frac{\mathrm{d}\bar h_s}{\mathrm{d}s}
    =
    \bar f_\theta(\bar h_s)
    \frac{\log S^N_{r_i,r_{i+1}}(X)}{r_{i+1}-r_i},
    \qquad s\in[r_i,r_{i+1}],
    \qquad \bar h_{r_i}=h_{r_i},
\end{equation}
and sets $h_{r_{i+1}}\approx \bar h_{r_{i+1}}$.

For Linear NCDEs and SLiCEs, the interval flow has the closed-form $\exp\left(\bar A_\theta\left(\log S^N_{s,t}(X)\right)\right)$, where $\bar A_\theta$ is obtained from the matrices $A^i_\theta$ by the same Lie-bracket construction, with vector-field brackets replaced by sign-flipped matrix commutators.
Each interval acts on the hidden state by a linear map and the hidden states can be computed in parallel using an associative scan \citep{walker2025structuredlinearcdesmaximally}. 
For SLiCEs, the efficiency of the chosen matrix structure is retained only when the Log-ODE construction remains inside that structure.
This is true for block-diagonal SLiCEs, since block-diagonal matrices are closed under products, commutators, and matrix exponentials \citep{walker2025structuredlinearcdesmaximally}.

We refer to the resulting structure-preserving Log-ODE variant of SLiCE as Log-SLiCE.
As shown in Section~\ref{sec:log_ode_exp}, Log-SLiCE reduces the time per training step by up to $3{,}000\times$ on time series with roughly $18{,}000$ observations when compared to standard NCDEs.
Figure~\ref{fig:log_linear_cde} illustrates the Log-SLiCE pipeline.
The next section constructs the map from data observations to log-signatures over intervals.

\begin{figure}
    \centering
    \includegraphics[width=\textwidth]{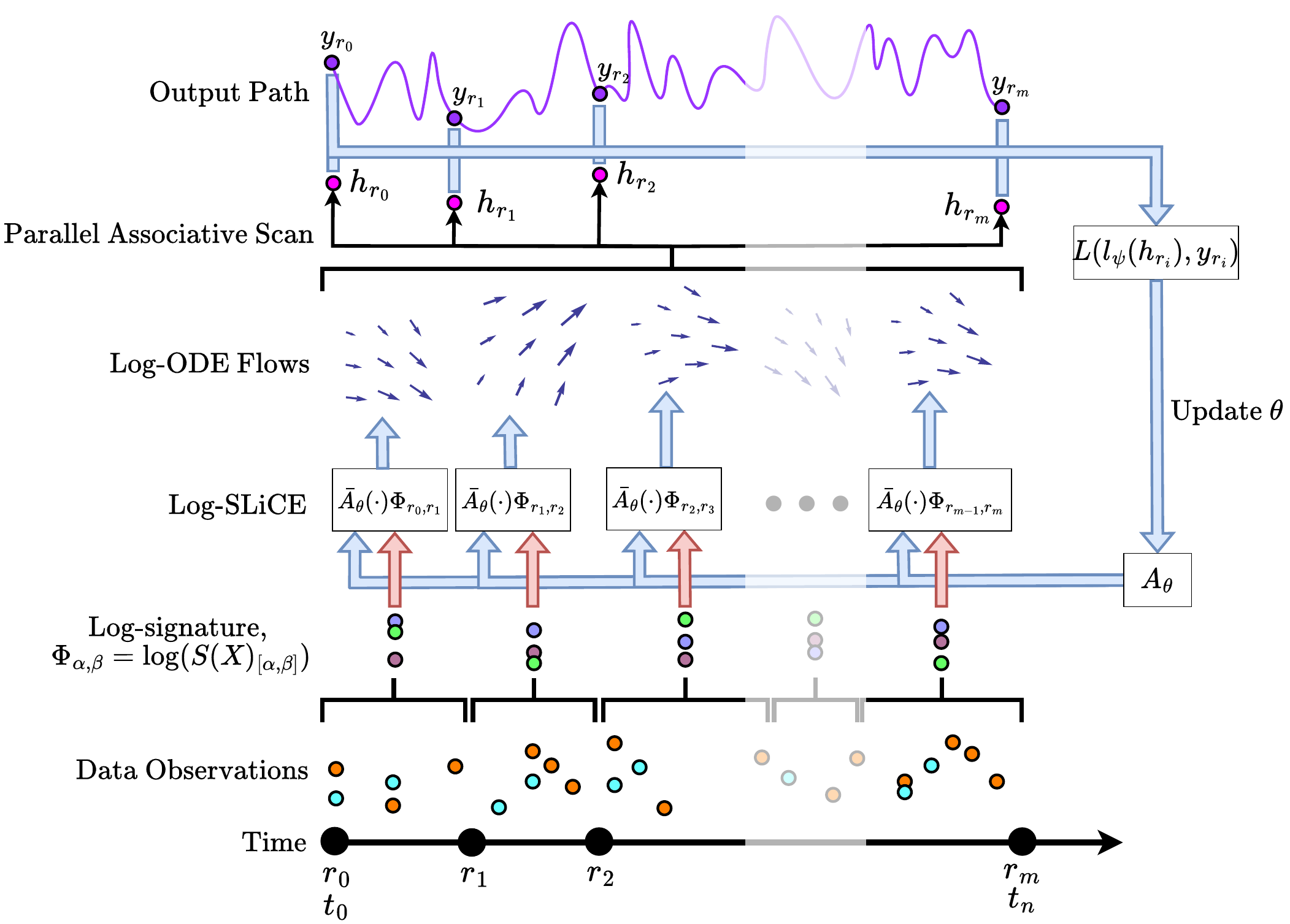}
    \caption{\textbf{Schematic of a Log-SLiCE.} Irregular observations are first summarised by interval log-signatures $\Phi_{\alpha,\beta}=\log(S(X)_{[\alpha,\beta]})$. The iterated Lie brackets of the matrices $A^i_\theta$, represented by $\bar A_\theta$, map each log-signature to a Log-ODE flow. These interval-wise hidden-state updates are composed using a parallel associative scan. The resulting hidden states $h_{r_i}$ are decoded to produce estimates of the output values $y_{r_i}$, and the loss $L(l_{\psi}(h_{r_i}),y_{r_i})$ is used to update the parameters.}
    \label{fig:log_linear_cde}
\end{figure}

\subsection{From observations to log-signatures}
\label{sec:procedure}

Suppose we have $d_{\mathrm{disc}}$ channels for which we observe values at discrete times in $[0,T]$.
We write the observation stream as
\begin{equation}
    x
    =
    ((t_0,s_0,x_0),\ldots,(t_n,s_n,x_n)),
    \qquad
    0\leq t_0<\cdots<t_n\leq T,
\end{equation}
where $s_i\subseteq\{1,\ldots,d_{\mathrm{disc}}\}$ is the non-empty set of channels observed at time $t_i$, and $x_i^k\in\mathbb R$ is the observed value of channel $k\in s_i$.
The stream can be irregular and asynchronous, because the gaps $t_{i+1}-t_i$ need not be equal and different subsets $s_i$ can be observed at each time.
Suppose we also have $d_{\mathrm{cont}}$ continuously observed channels for which we can access their value at any time, with time itself being the canonical example.
The goal is to define a faithful path embedding of these continuous variables and observation streams, together with a procedure for efficiently computing the embedded path's log-signature over any query interval, so that it can be used with Log-NCDEs.
Rather than first constructing an interpolated path and then computing its log-signature, we define the interval log-signatures directly.

We first choose the coordinates of the path space.
For an embedding to be injective, the construction must retain both the last recorded value of each discretely observed channel and whether a new observation has occurred.
We therefore use
\begin{equation}
    d_X = 2d_{\mathrm{disc}} + d_{\mathrm{cont}},
\end{equation}
where the first $d_{\mathrm{disc}}$ coordinates give the current recorded values of the discretely observed channels, the next $d_{\mathrm{disc}}$ coordinates record observation counts in each channel, and the final $d_{\mathrm{cont}}$ coordinates store continuously observed variables.
The count coordinates distinguish an observation that repeats the previous recorded value from the absence of a new observation in that channel.
Where this distinction is irrelevant, the count coordinates can be omitted.
The resulting embedding then intentionally identifies streams that differ only by the insertion or removal of value-preserving observations.
This is a modelling choice, with the invariance-computational tradeoff discussed in Appendices~\ref{app:invariances} and~\ref{app:limitations}.

We convert each observation into an increment in $\mathbb R^{d_X}$.
For $k\in s_i$, let $\operatorname{prev}(i,k)$ be the most recent index $j<i$ such that $k\in s_j$, i.e. the previous index where channel $k$ is observed.
We base-point augment the stream by setting $x_{\operatorname{prev}(i,k)}^k=0$ if no such index exists.
The event increment at time $t_i$ is
\begin{equation}
    \delta_i
    =
    \sum_{k\in s_i}
    \left[
        \left(
            x_i^k
            -
            x_{\operatorname{prev}(i,k)}^k
        \right)e_k
        +
    e_{d_{\mathrm{disc}}+k}
    \right]
    \in
    \mathbb R^{d_X},
    \label{eq:event_increment_method}
\end{equation}
where $\{e_i\}_{i=1}^{d_X}$ is the standard basis of $\mathbb{R}^{d_X}$.
Thus each observed channel contributes its change since its previous observation, together with a unit increment in its count coordinate. 

Each $\delta_i$ is then included into the truncated free Lie algebra.
By default, we place $\delta_i$ in degree one and set all higher-order components to zero,
$\Delta_i = (0,\delta_i,0,\ldots,0) \in \mathfrak L^N(\mathbb R^{d_X})$.
This is where additional local information can be incorporated.
For example, if an observation event includes precomputed signed-area terms, or if a local segment has known iterated integrals, these higher-order terms can be stored directly in $\Delta_i$. We provide an example of this for Brownian motion paths in Section~\ref{sec:bm_results}. The corresponding signature factor associated with the observation event is $E_i = \exp(\Delta_i) \in T^N(\mathbb R^{d_X})$.

Let $C:[0,T]\to\mathbb R^{d_{\mathrm{cont}}}$ be a continuous path of bounded-variation representing the continuously observed variables, embedded in $\mathbb R^{d_X}$ as $\widetilde C_t=(0,\ldots,0,C_t)$, and write $\Gamma_{u,v}=S^N_{u,v}(\widetilde C)$ and $B_C=\exp(\tilde C_0)$.
For a query interval $[\alpha,\beta)$ containing events $\alpha\leq t_{i_1}<\cdots<t_{i_m}<\beta$, the signature over the interval is defined by
\begin{equation}
    G_{\alpha,\beta}(x,C)
    =
    \begin{cases}
    B_C\otimes\Gamma_{\alpha,t_{i_1}}
    \otimes E_{i_1}
    \otimes \Gamma_{t_{i_1},t_{i_2}} \otimes E_{i_2}
    \otimes \cdots
    \otimes E_{i_m}
    \otimes \Gamma_{t_{i_m},\beta},
    & \alpha=0,\\
    \Gamma_{\alpha,t_{i_1}}
    \otimes E_{i_1}
    \otimes \Gamma_{t_{i_1},t_{i_2}} \otimes E_{i_2}
    \otimes \cdots
    \otimes E_{i_m}
    \otimes \Gamma_{t_{i_m},\beta},
    & \alpha>0,
    \end{cases}
\end{equation}
with $B_C$ base-point augmenting $C$. 
Events at a right endpoint are assigned to the next interval, except that events at $T$ are included in the final interval.
The interval log-signature is
\begin{equation}
    \Phi_{\alpha,\beta}(x,C)
    =
    \log G_{\alpha,\beta}(x,C)
    \in
    \mathfrak L^N(\mathbb R^{d_X}).
    \label{eq:interval_log_signature_method}
\end{equation}
The Log-ODE method requires interval log-signatures on a partition $\pi = \{0=r_0<r_1<\cdots<r_M=T\}$, which we denote by $\varphi_\pi(x,C) = \left( \Phi_{r_0,r_1}(x,C), \ldots, \Phi_{r_{M-1},r_M}(x,C) \right)$. 
The construction is depicted schematically in Figure~\ref{fig:data_to_log_sigs}.

Once the event increments have been formed, each interval can be processed independently and its factors composed in parallel using an associative scan, allowing for efficient computation.
Furthermore, the construction is online, since for any $\pi$, the information needed to compute each $\Phi_{r_i,r_{i+1}}(x,C)$ depends only on information up to time $r_{i+1}$. 
Thus the only remaining desired property is that these interval summaries arise from a faithful path embedding, since NCDEs and Log-NCDEs are universal models on path space.
We therefore proceed by constructing a continuous path \(\Psi(x,C)\) whose interval log-signatures equal \(\Phi_{\alpha,\beta}(x,C)\) and which satisfies the following proposition.

\begin{proposition}[Faithfulness (informal)]
\label{prop:faithfulness}
Assume the observation-count coordinates are included and that $C$ contains time as a channel. Then the map $(x,C)\to\Psi(x,C)$ is faithful.
\end{proposition}

The construction of \(\Psi(x,C)\) and formal statement and proof of Proposition~\ref{prop:faithfulness} are given in Appendix~\ref{app:proof_faithfulness}.
As with the rectilinear construction \citep{morrill2022interpolation}, each observation event is assigned an interval of auxiliary time along which value and count coordinates move continuously while physical time is held fixed. Continuously observed variables contribute their signature factors over the physical-time gaps between events.
Since signatures are invariant under reparametrisations of the auxiliary parameter, the durations assigned to these event intervals do not affect the resulting physical-time interval signatures. 

\begin{figure}
    \centering
    \includegraphics[width=1.0\textwidth]{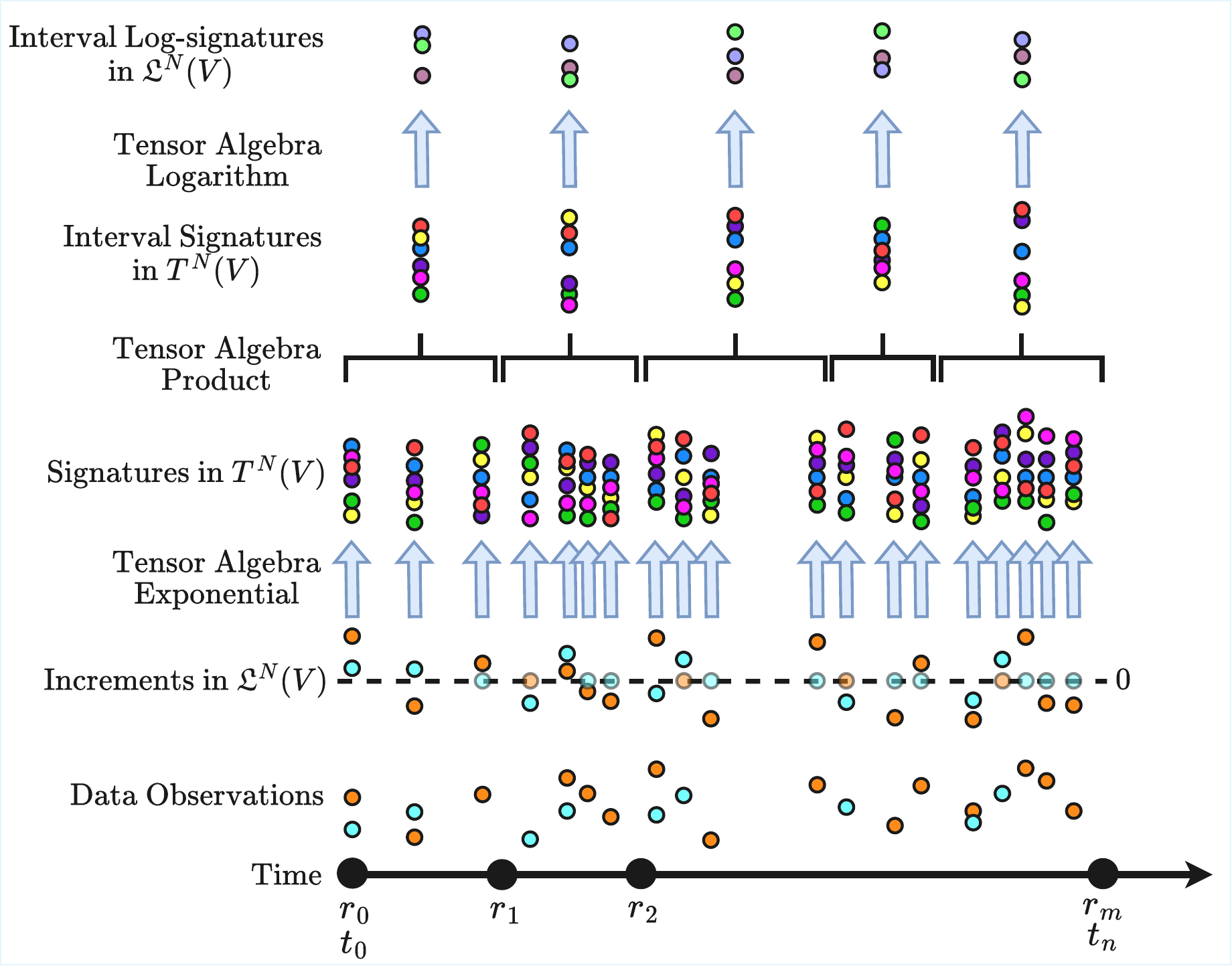}
    \caption{\textbf{Constructing interval log-signatures from irregular observations.}
    Each observation is converted into a Lie-algebra increment containing value increments and observation-count increments.
    Continuously observed variables, such as time, contribute signature factors over gaps between observations.
    Event and gap factors are composed over each query interval using a parallel associative scan with the tensor product, then mapped back to the free Lie algebra by the tensor logarithm.
    This gives interval log-signatures without interpolating the discretely observed channels.}
    \label{fig:data_to_log_sigs}
\end{figure}

\section{Experiments}
\label{sec:experiments}

We evaluate our proposed embedding in four settings.
First, we use EigenWorms, the longest dataset in the UEA multivariate time-series classification archive, to test the effect of the Log-ODE method on accuracy, speed, and GPU memory.
Second, we test prediction on a synthetic sinusoidal system with output grids independent of input observation times.
Third, we test prediction on a Brownian controlled system whose solution depends on second-order information.
Fourth, we evaluate robustness on six UEA classification datasets under input dropping.
The empirical scope is discussed further in Appendix~\ref{app:limitations}.

\subsection{Effect of the Log-ODE Method}
\label{sec:log_ode_exp}

Table~\ref{tab:log_method_comparison_ew} compares three NCDE models with and without the Log-ODE method on EigenWorms, holding all hyperparameters fixed to those selected by \citet{Walker2024LogNCDE} for NCDE and Log-NCDE.
EigenWorms has $17{,}984$ observations per series, making it a direct test of whether interval-level computation reduces the cost of long input sequences without sacrificing accuracy.
Timing and memory measurements were run on an NVIDIA H100 GPU with a batch size of $1$.
Further experimental details are given in Appendix~\ref{app:uea-eigenworms}.
On this benchmark, every Log-ODE variant is faster and more accurate.
Overall, combining the Log-ODE method with a SLiCE increases test accuracy by $11.1$ percentage points over a standard NCDE, while reducing the time per training step by a factor of almost $3{,}000$.
This cuts the time for $10{,}000$ training steps from $3$ days to $1.5$ minutes.
Within each family, the Log-ODE method speeds up NCDEs by approximately $11.5\times$, LNCDEs by approximately $5.6\times$, and BD-SLiCE by approximately $2.7\times$.
\begin{table}[t]
\centering
\small
\setlength{\tabcolsep}{10pt}
\renewcommand{\arraystretch}{1.2}
\caption{\textbf{NCDE models with and without the Log-ODE method on EigenWorms.}
Results marked~$^\dagger$ are from \citet{Walker2024LogNCDE}; results marked $^\ddagger$ are from \citet{walker2025structuredlinearcdesmaximally}.}
\begin{tabular}{cc|ccc}
\toprule
Model family & Log-ODE & Test Acc. & Time per step (s) & GPU Mem. (MB) \\
\midrule
NCDE
  & $\times$     & $75.0 \pm 3.9$$^\dagger$ & $26.02$ & $3484$ \\
  & $\checkmark$ & $85.6 \pm 5.1$$^\dagger$ & $2.263$ & $3494$ \\
\midrule
LNCDE
  & $\times$     & $87.2 \pm 5.2$ & $0.100$ & $13730$ \\
  & $\checkmark$ & $88.3 \pm 3.7$$^\ddagger$ & $0.018$ & $3492$ \\
\midrule
BD-SLiCE
  & $\times$     & $81.1 \pm 6.7$ & $0.024$ & $3494$ \\
  & $\checkmark$ & $86.1 \pm 3.5$$^\ddagger$ & $0.009$ & $3494$ \\
\bottomrule
\end{tabular}
\vspace{0.15cm}
\label{tab:log_method_comparison_ew}
\end{table}

\subsection{Synthetic coupled data}\label{sec:sinusoid}

For a synthetic benchmark, we consider the multivariate system
\begin{align}
\label{eqn:syn_dynamics}
z(t) &= \omega t+\phi, \nonumber\\
x_i(t) &= A_i\sin\bigl(z(t)+\delta_i\bigr), \qquad i=1,\ldots,d,
\end{align}
where $z$ is a shared latent phase which gives cross-channel structure. The parameters \(\omega\) and \(\phi\) are independently drawn per sample.
The parameters $A_i$ and $\delta_i$ are sampled independently per channel. We fix $d=2$. For further information on the set-up see Appendix~\ref{app:sinusoid}.

The observations are generated under four sampling regimes. 
The channels may be observed synchronously and regularly, synchronously and irregularly, asynchronously and irregularly, or at significantly different frequencies, where one channel is much sparser than the other.

For each sample, we draw a random query partition
\(
    0=q_0<q_1<\cdots<q_m=T,
\)
where $m$ varies between samples. 
We test whether models can learn the value at the end of each query interval $[q_k,q_{k+1})$, i.e. $x(q_{k+1})$, constructed from~\eqref{eqn:syn_dynamics}. This is not next-observation prediction since query endpoints need not coincide with observation times. 
The task therefore tests whether the model can use interval summaries to predict on an output grid chosen independently of the input sampling grid.

We first train one model for each data observation regime. Illustrative observation patterns and sample model outputs are shown in Figures~\ref{fig:example_sync_regular}--\ref{fig:example_async_sparse} of Appendix~\ref{app:sinusoid}. We then evaluate these pre-trained models across the other regimes to assess cross-regime generalisation. Figure~\ref{fig:sinusoid} reports the mean and standard deviation of this analysis. A model trained on synchronously observed, regularly sampled data struggles to generalise to asynchronously observed data in which one channel is sampled much more sparsely. This suggests that the model learns shortcuts when trained only on regularly sampled data. Conversely, because the asynchronous sparse regime is the hardest setting, a model trained on it remains robust across all other regimes. For irregularly sampled data, there is little difference between synchronous and asynchronous observations. This agrees with our view of faithful embeddings, since both regimes contain similar quantities of information and the model can learn to use it when it is faithfully represented. Figure~\ref{fig:sinusoid_loss_traj} in Appendix \ref{app:sinusoid} shows that their training curves are also very similar. 


\begin{figure}[t]
    \centering
    \includegraphics[width=\linewidth]{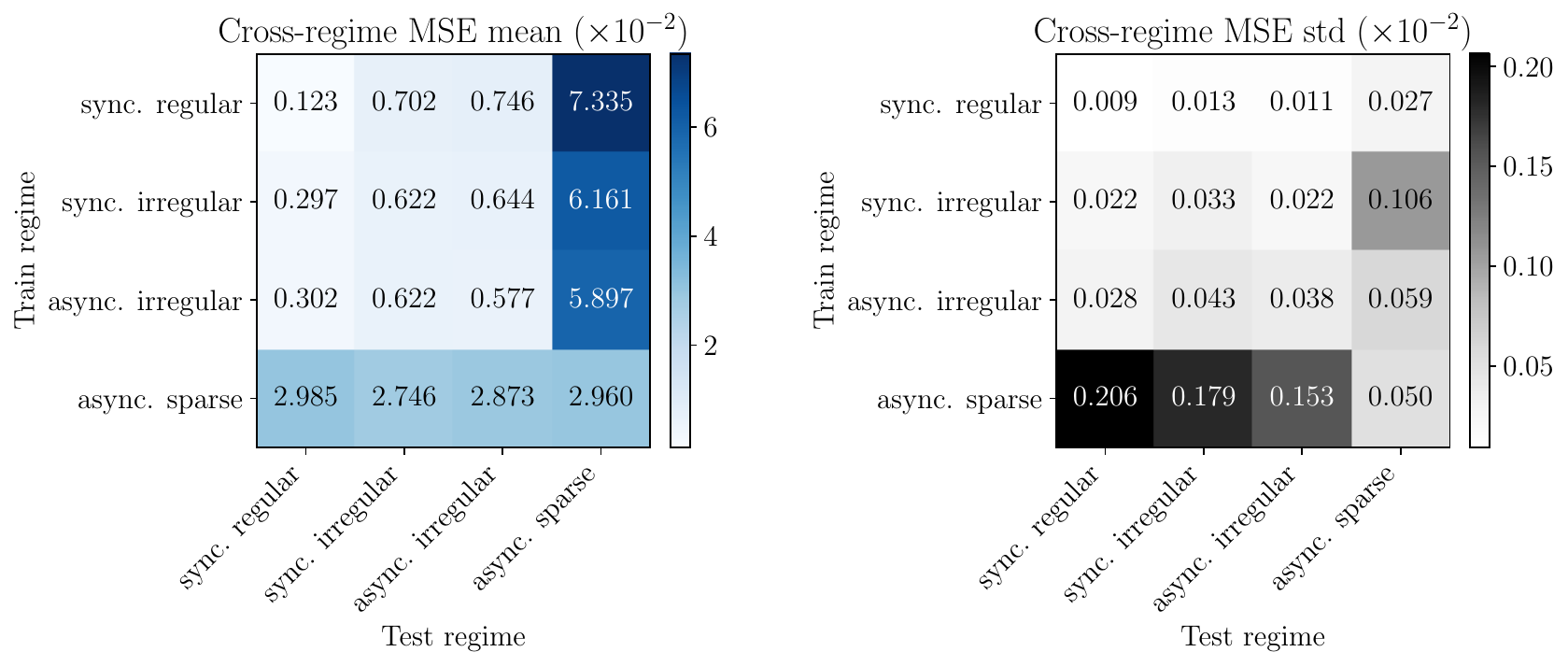}
    \caption{\textbf{Heatmap showing how well trained sinusoid models generalise across data sampling regimes.} The mean and standard deviation are computed over $5$ random seeds.
    }
    \label{fig:sinusoid}
\end{figure}

\subsection{Linear system driven by Brownian motion}
\label{sec:bm_results}

Next we test whether the model can exploit higher order information if it is available. We use precomputed Brownian rough paths, with the L\'{e}vy area computed as described by \citet{jelincic2025generative}. 
Each sample consists of a $4$-dimensional Brownian stream on a fine regular grid ($2048$ observations between $[0, 1]$), represented by log-signatures up to level $2$, giving an input dimension of $10$. 
These Brownian paths are used to drive a Stratonovich linear controlled system
\begin{equation}
    \mathrm{d}X_t
    =
    V_1X_t\circ\mathrm{d}W^1_t
    +
    V_2X_t\circ\mathrm{d}W^2_t,
\end{equation}
where $V_1,V_2\in\mathbb{R}^{r\times r}$ are fixed non-commuting matrices. Note that the CDE is only being controlled by two dimensions of the Brownian stream, therefore part of the task is to also learn which streams are redundant. We fix the dimension of the output path $r=2$.
The non-commutativity ensures that the solution depends not only on first-level Brownian increments, but also on second-level Lie bracket terms.
The model predicts the state at the next query endpoint,
\(
    y_k=X_{k+1},\: k=0,\ldots,m-1.
\)

In Figure~\ref{fig:bm_log}, we compare how our model performs when L\'{e}vy area information is provided compared with using only first level information, against the number of query intervals used.
The solution depends intrinsically on second-level information, and we see that lower MSE is achieved with the higher order information. As the number of query intervals increase, the truncation error of using only first level decreases in magnitude, hence the test MSE decreases for the model trained on level $1$. This highlights the flexibility of our approach in encompassing this additional information where it is available.
Precise MSE averaged over $5$ seeds can be seen in Table~\ref{tab:intervals_vs_error} in Appendix \ref{app:brownian-rde}. 

\begin{figure}
    \centering
    \includegraphics[width=0.6\linewidth]{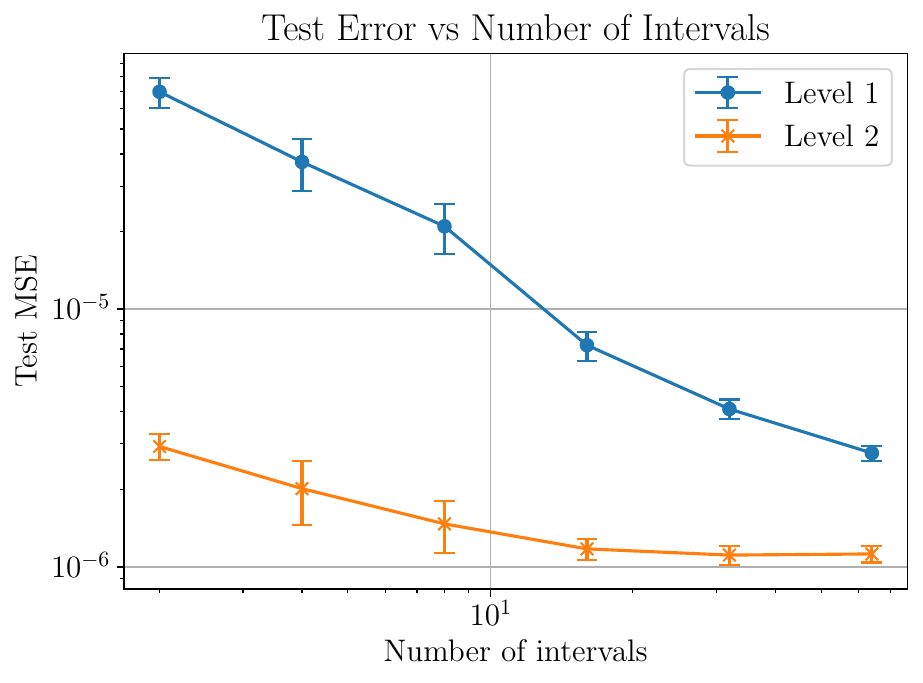}
    \caption{\textbf{Average test MSE against number of query intervals used for linear system driven by Brownian motion over $\boldsymbol{5}$ seeds.} As the system is driven by area information, we see a much lower MSE for when using level $2$ log-signatures as we would expect, especially for fewer query intervals.}
    \label{fig:bm_log}
\end{figure}

\subsection{UEA input dropping}

We evaluate robustness under input dropping on the six UEA time-series classification datasets chosen by \citet{Walker2024LogNCDE}.
The comparison includes four structured recurrent baselines, S5 \citep{S5}, LRU \citep{orvieto2023resurrecting}, S6, and Mamba \citep{gu2024mamba}, and an NCDE using Hermite cubic interpolation with backward differences \citep{kidger2020neuralcde}.
We test our proposed embedding using a Log-NCDE \citep{Walker2024LogNCDE}, diagonal Log-SLiCE (D-SLiCE), and block-diagonal Log-SLiCE (BD-SLiCE) \citep{walker2025structuredlinearcdesmaximally}.
For path-based models, we keep the original query windows fixed and recompute interval log-signatures from the remaining observations.
Full details and per-dataset results are given in Appendix~\ref{app:uea-experiments}.

\begin{table*}[t]
  \centering
  \small
  \setlength{\tabcolsep}{4.5pt}
  \caption{\textbf{Results on the six UEA datasets.}
  Average test accuracy (in \%) and average dataset rank with no dropping and $30\%$, $70\%$, and $95\%$ input dropping.
  Lower rank is better.
  No-drop values marked $^\dagger$ are from \citet{Walker2024LogNCDE}; no-drop values marked $^\ddagger$ are from \citet{walker2025structuredlinearcdesmaximally}.}
  \begin{tabular}{lcc|cc|cc|cc}
  \toprule
  & \multicolumn{2}{c|}{No drop} & \multicolumn{2}{c|}{30\% drop} & \multicolumn{2}{c|}{70\% drop} & \multicolumn{2}{c}{95\% drop} \\
  Model
  & Av.\ Acc. & Av.\ Rank
  & Av.\ Acc. & Av.\ Rank
  & Av.\ Acc. & Av.\ Rank
  & Av.\ Acc. & Av.\ Rank \\
  \midrule
  S5 & $61.8$$^\dagger$ & $4.67$$^\dagger$ & $63.4$ & $3.25$ & $60.1$ & $4.83$ & $56.7$ & $5.33$ \\
  LRU & $61.7$$^\dagger$ & $4.50$$^\dagger$ & $60.9$ & $5.58$ & $61.0$ & $5.00$ & $60.8$ & \underline{$4.17$} \\
  S6 & $62.0$$^\dagger$ & $5.00$$^\dagger$ & $58.4$ & $6.50$ & $61.4$ & $4.00$ & $58.9$ & $5.17$ \\
  Mamba & $58.6$$^\dagger$ & $6.67$$^\dagger$ & $59.5$ & $4.83$ & $60.6$ & $4.00$ & $57.6$ & $6.17$ \\
  NCDE & $60.2$$^\dagger$ & $5.33$$^\dagger$ & $59.6$ & $5.92$ & $58.9$ & $6.42$ & $58.7$ & $5.33$ \\
  Log-NCDE & $\mathbf{64.3}$$^\dagger$ & \underline{$3.00$}$^\dagger$ & $\mathbf{64.2}$ & \underline{$3.08$} & $\mathbf{62.7}$ & $4.67$ & $60.9$ & \underline{$4.17$} \\
  D-SLiCE & $61.7$$^\ddagger$ & $4.50$$^\ddagger$ & $61.7$ & $3.83$ & $62.2$ & \underline{$3.83$} & \underline{$62.2$} & $\mathbf{2.83}$ \\
  BD-SLiCE & \underline{$64.0$}$^\ddagger$ & $\mathbf{2.33}$$^\ddagger$ & \underline{$63.7$} & $\mathbf{3.00}$ & \underline{$62.5$} & $\mathbf{3.25}$ & $\mathbf{62.5}$ & $\mathbf{2.83}$ \\
  \bottomrule
  \end{tabular}
  \label{tab:uea_drop_summary}
\end{table*}

\begin{table}[t]
\centering
\caption{\textbf{Matched input-representation ablation.} Each row replaces BD-SLiCE in Table~\ref{tab:uea_drop_summary} by BD-SLiCE with the given input representation and re-ranks all eight models on each dataset. Entries give its average dataset rank and, in parentheses, its position among the eight models ($=$ denotes a tie). Avg.\ averages the dataset ranks over the three dropping levels.}
\label{tab:representation-ablation}
\begin{tabular}{lcccc}
\toprule
Input representation & $30\%$ drop & $70\%$ drop & $95\%$ drop & Avg. \\
\midrule
Event-increment embedding (ours) & $\mathbf{3.00}$ ($1$) & $3.25$ ($1$)          & $\mathbf{2.83}$ ($=1$) & $\mathbf{3.03}$ \\
Cubic Hermite interpolation      & $3.58$ ($3$)          & $\mathbf{3.00}$ ($1$) & $3.58$ ($2$)           & $3.39$ \\
Gaussian-process imputation      & $4.33$ ($4$)          & $4.50$ ($5$)          & $3.58$ ($2$)           & $4.14$ \\
Global-mean imputation           & $5.17$ ($5$)          & $3.75$ ($2$)          & $3.33$ ($2$)           & $4.08$ \\
\bottomrule
\end{tabular}
\end{table}

Table~\ref{tab:uea_drop_summary} shows that the continuous-time models using our embedding are competitive across all dropping levels. 
Log-NCDE gives the best macro-average accuracy at no drop, $30\%$ drop, and $70\%$ drop. 
BD-SLiCE gives the best macro-average accuracy at $95\%$ drop and the best or tied-best average rank at every dropping level. 
Its accuracy changes from $64.0\%$ with no dropping to $62.5\%$ with $95\%$ dropping, indicating robustness under severe sparsification. To isolate the input representation, Table~\ref{tab:representation-ablation} repeats the evaluation with BD-SLiCE while changing only its input representation, holding the architecture, hyperparameters, data splits, and seeds fixed. Our embedding is the only representation that places BD-SLiCE first or joint first at every dropping level, and it attains the best mean rank. Cubic Hermite interpolation attains a slightly better rank at $70\%$ drop but not at $30\%$ or $95\%$, and the imputation-based alternatives are less competitive overall. 

\section{Conclusion}

This paper argues that continuous-time representations should preserve observed data rather than attempt to reconstruct unobserved behaviour.
This principle is formalised through a general universality transfer result.
For Hausdorff data and model input spaces, any continuous injective embedding transfers universality on compact sets from the model input space back to the data space. 
Thus, for a universal downstream model, an embedding need only be continuous and injective. 
Guided by this, we introduced a faithful embedding for irregular and partially observed streams tailored to Log-NCDEs.
The construction records observations as tensor-algebra factors, composes them over arbitrary query intervals, and maps the result to interval log-signatures for use with the Log-ODE method.
This approach avoids interpolation of observed variables, supports online prediction, decouples input observation times from output query times, and is computable in parallel over intervals and parallel-in-time within intervals.
Experiments on synthetic controlled systems and UEA classification tasks show that the resulting models are accurate, efficient, and robust under irregularity, asynchrony, and severe input dropping.


\section*{Acknowledgements}

\small 
Benjamin Walker is supported by UK Research and Innovation (UKRI) through the Engineering and Physical Sciences Research Council (EPSRC) via Programme Grant [Grant No.\ UKRI1010: High order mathematical and computational infrastructure for streamed data that enhance contemporary generative and large language models] and CIMDA@Oxford, part of the AIR@InnoHK initiative funded by the Innovation and Technology Commission, HKSAR Government.
Lingyi Yang is supported by the Maths Degrees for the Future Programme Grant MDFF 2024-08 awarded by the London Mathematical Society.
Terry Lyons is supported by UK Research and Innovation (UKRI) through the Engineering and Physical Sciences Research Council (EPSRC) via Programme Grants [Grant No.\ UKRI1010: High order mathematical and computational infrastructure for streamed data that enhance contemporary generative and large language models], [Grant No.\ EP/S026347/1: Unparameterised multi-model data, high order signatures and the mathematics of data science], [Grant No.\ EP/Y028872/1: Mathematical Foundations of Intelligence: An Erlangen Programme for AI], and the UKRI AI for Science award [Grant No.\ UKRI2385: Creating Foundational Benchmarks for AI in Physical and Biological Complexity]. Terry Lyons is also supported by The Alan Turing Institute under the Defence and Security Programme (funded by the UK Government) and through the provision of research facilities; by the UK Government; and through CIMDA@Oxford, part of the AIR@InnoHK initiative funded by the Innovation and Technology Commission, HKSAR Government. The authors acknowledge support from His Majesty's Government in the development of this research. 
The authors would like to acknowledge the use of the University of Oxford Advanced Research Computing (ARC) facility in carrying out this work. https://doi.org/10.5281/zenodo.22558.
Continuous Intelligence Limited is supported by SPRIND through their Next Frontier AI Challenge.

\newpage
\bibliographystyle{plainnat}
\bibliography{references}

\newpage
\appendix

\section{Additional Details on Faithful Embeddings}
\label{app:faithful-embeddings}

This appendix gives the technical details behind the faithful-embedding criterion used in Section~\ref{sec:theory}.
The key point is that a continuous injective embedding preserves compact sets up to homeomorphism.
Consequently, a universal function class on the embedded space can be pulled back to a universal function class on the original data space.
We first prove the transfer result, then state the corresponding non-injectivity obstruction, and finally record the standard signature-universality consequence.

\subsection{Universality transfer}

\begin{lemma}[Universality Transfer]
\label{lem:transfer-app}
Let $\mathcal{X}$ and $\mathcal{P}$ be Hausdorff topological spaces, and let $\varphi : \mathcal{X} \to \mathcal{P}$ be a continuous injection.
Suppose $\mathcal{F} \subseteq C(\mathcal{P}; \mathbb{R})$ is a family of continuous real-valued functions such that for every compact $K \subseteq \mathcal{P}$, the restrictions $\mathcal{F}|_K$ are dense in $C(K; \mathbb{R})$ in the uniform topology.
Then for every compact $K_0 \subseteq \mathcal{X}$, the restrictions $(\mathcal{F} \circ \varphi)|_{K_0}$ are dense in $C(K_0; \mathbb{R})$ in the uniform topology, where
\[
    \mathcal{F} \circ \varphi := \{ f \circ \varphi : f \in \mathcal{F} \}.
\]
\end{lemma}

\begin{proof}
Let $K_0 \subseteq \mathcal{X}$ be compact, let $u \in C(K_0; \mathbb{R})$, and let $\varepsilon > 0$.
Set $K := \varphi(K_0) \subseteq \mathcal{P}$.
The continuous image of a compact set is compact, so $K$ is compact.
The restriction $\varphi|_{K_0} : K_0 \to K$ is a continuous bijection from a compact space to a Hausdorff space, hence a homeomorphism.
In particular, $(\varphi|_{K_0})^{-1}$ is continuous.

Let
\[
    g := u \circ (\varphi|_{K_0})^{-1} : K \to \mathbb{R}.
\]
Then $g \in C(K; \mathbb{R})$.
By assumption, $\mathcal{F}|_K$ is dense in $C(K; \mathbb{R})$ in the uniform topology, so there exists $h \in \mathcal{F}$ with
\[
    \sup_{p \in K} |g(p) - h(p)| < \varepsilon.
\]
For every $x \in K_0$,
\[
    g(\varphi(x))
    =
    u((\varphi|_{K_0})^{-1}(\varphi(x)))
    =
    u(x).
\]
Therefore
\[
    \sup_{x \in K_0} |u(x) - h(\varphi(x))|
    =
    \sup_{x \in K_0} |g(\varphi(x)) - h(\varphi(x))|
    =
    \sup_{p \in K} |g(p) - h(p)|
    < \varepsilon.
\]
Since $h \circ \varphi \in \mathcal{F} \circ \varphi$, this proves that $(\mathcal{F} \circ \varphi)|_{K_0}$ is dense in $C(K_0; \mathbb{R})$.
\end{proof}

The proof uses only the standard fact that a continuous bijection from a compact space to a Hausdorff space is a homeomorphism.
In the main text, $\mathcal{X}$ is the original data space, $\mathcal{P}$ is the model input space, and $\varphi$ is the embedding.

\subsection{Non-injectivity obstruction}

We now record the complementary failure mode.
If an embedding identifies two distinct data objects, then every downstream function of the embedding identifies them as well.
Thus no downstream class can approximate continuous functions that separate those data objects.

\begin{proposition}[Non-injective implies non-universal]
\label{prop:non-universality}
Let $\mathcal{X}$ and $\mathcal{P}$ be Hausdorff topological spaces, let $\varphi:\mathcal{X} \to \mathcal{P}$ be continuous, and let $\mathcal{F} \subseteq C({\mathcal{P}},\mathbb{R})$. If $\varphi$ is not injective on a compact $K_0 \subseteq \mathcal{X}$, then $(\mathcal{F} \circ \varphi)|_{K_0}$ is not dense in $C(K_0; \mathbb{R})$ in the uniform topology, where
\[
    \mathcal{F} \circ \varphi := \{f \circ \varphi : f \in \mathcal{F}\}.
\]
\end{proposition}

\begin{proof}
By assumption, there exist distinct $x_1, x_2 \in K_0$ with $\varphi(x_1) = \varphi(x_2)$.
Since $K_0$ is compact Hausdorff, it is normal.
Hence there exists $f \in C(K_0; \mathbb{R})$ with $f(x_1) \neq f(x_2)$.
Set
\[
    \varepsilon = \frac{|f(x_1) - f(x_2)|}{3}.
\]

Let $h \in \mathcal{F}$ be arbitrary.
Since $\varphi(x_1) = \varphi(x_2)$, we have $h(\varphi(x_1)) = h(\varphi(x_2))$, and therefore
\[
    |f(x_1) - h(\varphi(x_1))|
    +
    |f(x_2) - h(\varphi(x_2))|
    \geq
    |f(x_1) - f(x_2)|
    =
    3 \varepsilon.
\]
Hence one of the two summands is at least $3\varepsilon/2 > \varepsilon$, so
\[
    \sup_{x \in K_0} |f(x) - h(\varphi(x))| > \varepsilon.
\]
Since $h \in \mathcal{F}$ was arbitrary, no element of $\mathcal{F} \circ \varphi$ approximates $f$ to within $\varepsilon$ on $K_0$.
Thus $(\mathcal{F} \circ \varphi)|_{K_0}$ is not dense in $C(K_0; \mathbb{R})$.
\end{proof}




Lemma~\ref{lem:transfer-app} and Proposition~\ref{prop:non-universality} give the criterion used in the main text.
Continuous and injective embeddings preserve universality on compact sets.
Non-injective embeddings cannot preserve universality, because they discard information about the original data.

\subsection{Signature representations of data}

Classical signature universality for paths has the same form as Lemma~\ref{lem:transfer-app}, after restricting to a class of paths on which the signature map is injective, or equivalently after quotienting by tree-like equivalence \citep{hambly2010uniqueness, BOEDIHARDJO2016720}. Recent work has also extended signature universality results to stochastic settings \citep{cuchiero2023universalapproximationtheoremscontinuous, ceylan2026universalapproximationsignaturesnongeometric, chevyrev2026orthogonalpolynomialspathspace}.
The following corollary records the same schema with paths replaced by data streams.
Any continuous injection from data into the tensor algebra inherits universality through linear functionals.

\begin{corollary}[Universality of signature representations of data]
\label{cor:sig_universality}
Let $d,m,N\in\mathbb{N}$, let
\[
    T(\!(\mathbb{R}^{m})\!)=\prod_{k=0}^{\infty}(\mathbb{R}^{m})^{\otimes k}
\]
be equipped with the product topology, and let $G \subseteq T(\!(\mathbb{R}^{m})\!)$ be the group-like elements with the subspace topology.
Let $\iota : (\mathbb{R}^d)^{N} \to G$ be a continuous injection.
Then for every compact $K_0 \subseteq (\mathbb{R}^d)^{N}$, every continuous $f : K_0 \to \mathbb{R}$, and every $\varepsilon > 0$, there exists a continuous linear functional $\ell$ on $T(\!(\mathbb{R}^{m})\!)$ depending on only finitely many tensor-algebra 
coordinates with
\[
    \sup_{x \in K_0} \bigl| f(x) - \ell(\iota(x)) \bigr| < \varepsilon.
\]
\end{corollary}

\begin{proof}
Apply Lemma~\ref{lem:transfer-app} with $\mathcal{X} = (\mathbb{R}^d)^{N}$, $\mathcal{P} = G$, $\varphi = \iota$, and $\mathcal{F}$ the family of restrictions to $G$ of finitely supported continuous linear functionals on $T(\!(\mathbb{R}^{m})\!)$.
The density hypothesis on $\mathcal{F}$ follows from the shuffle identity and the real Stone--Weierstrass theorem, since these functionals contain the constants, separate points of $G$, and form an algebra on $G$ \citep{reutenauer1993free, StoneWeierstrass, Lyons2014}.
\end{proof}

This is the same mechanism used in signature-based universality results.
The only additional requirement is that the representation must separate the original data objects.
For paths, this is the injectivity of the signature after removing tree-like equivalence.
For data streams, it is the injectivity of the chosen embedding $\iota$.

Since linear functionals of signatures are universal on compact path classes once injectivity is established, truncated signature features provide expressive representations for statistical learning tasks \citep{levin2016learningpastpredictingstatistics}. Signature-based kernels extend this viewpoint to kernel methods on sequential data \citep{kiraly2019kernels}, while subsequent work develops learning directly with the full infinite-dimensional signature kernel \citep{salvi2021signature, salvi2021rough, lemercier2021distribution, lemercier2021siggpde, manten2025signature}.

Signatures have a number of other useful properties, including that their iterated-integral coordinates are unchanged by translations of the path and by increasing reparameterisations of time \citep{Chen1954Iterated}. Together with universality, these properties have led to their use across a wide range of domains, including healthcare \citep{Perez_Arribas2018-xp, Moore2019-tp, wang19e_interspeech, Morrill2019, cohen2023subtle, Vauvelle2022NeuralSignatureEHR}, quantitative finance \citep{Arribas2018DerivativesPU, arribas2020sigsdes, horvath2023optimal, cirone2025rough, cohen2023nowcasting}, information theory \citep{salvi2023structure, shmelev2024sparse}, cybersecurity \citep{cochrane2021sk}, and computational neuroscience \citep{holberg2024exact}.

Neural controlled differential equations extend this framework to learnable continuous-time models driven by paths, with applications including counterfactual prediction in healthcare \citep{seedat2022continuous}, economic nowcasting \citep{seonkyu2024bridging}, survival prediction \citep{zeng2025trajsurv}, anomaly detection for driving assistance \citep{lee2024gdflowanomalydetectionncdebased}, dynamic graph learning \citep{qin2023learningdynamicgraphembeddings, berndt2025permutation}, and traffic forecasting \citep{choi2022STGNCDE, choi2023graphneuralroughdifferential}.

\subsection{Proof of faithfulness}
\label{app:proof_faithfulness}

We give a formal version of Proposition~\ref{prop:faithfulness}.
The input consists of a continuously observed path \(C\) and an observation stream \(x\).
We assume throughout that \(C\) contains physical time as one of its channels.

Let \(\tau\in\{1,\ldots,d_{\mathrm{cont}}\}\) denote the continuous coordinate corresponding to physical time.
Thus
\[
    C_t^\tau=t,
    \qquad
    t\in[0,T].
\]
Let \(\mathcal C_{\mathrm{BV}}^{\tau}\) be the space of bounded-variation paths
\[
    C:[0,T]\to\mathbb R^{d_{\mathrm{cont}}}
\]
with $C_t^\tau=t, \:\forall t\in[0, T]$, equipped with the uniform topology.

For a fixed event pattern
\[
    \sigma=(s_0,\ldots,s_{m-1}),
\]
let \(\mathcal X_\sigma\) be the corresponding stream component.
On this component, the number of events and the observed sets are fixed, while the event times and local event records vary.
We equip each \(\mathcal X_\sigma\) with the standard product topology.
We equip
\[
    \mathcal X
    =
    \coprod_\sigma \mathcal X_\sigma
\]
with the disjoint-union topology.
The joint input space is
\[
    \mathcal D
    =
    \mathcal C_{\mathrm{BV}}^{\tau}
    \times
    \mathcal X.
\]

For a stream with \(m\) events, set
\[
    I_m=[0,T+m+1].
\]
The target path space is
\[
    \mathcal P
    =
    \coprod_{m\geq0}
    C(I_m,\mathbb R^{d_X}),
\]
with the disjoint-union topology and the uniform topology on each component.
The realised paths below lie in the bounded-variation subspace of \(\mathcal P\).

We write \(o_i\) for the local record inserted at event \(i\).
For each observed set \(s\), assume that local records are included through a continuous injective map
\[
    \iota_s:o\mapsto\Delta\in\mathfrak L^N(\mathbb R^{d_X}).
\]
We assume that \(\Delta\) uses only the first \(2d_{\mathrm{disc}}\) coordinates.
We also assume that its degree-one count coordinates satisfy
\[
    \pi_1(\Delta)^{d_{\mathrm{disc}}+k}
    =
    \mathbf 1_{\{k\in s\}}.
\]
Thus the count coordinates identify the observed set.
For higher-order local information, the inserted object must be expressed in log-signature coordinates.
If a local segment has a known truncated signature \(H\), then the element stored as the event factor is
\[
    \Delta=\log H.
\]

Assume that each local factor has a continuous bounded-variation realisation.
For each observed set \(s\), assume there is a continuous map
\[
    o\mapsto \eta_s(o)
\]
from local records into bounded-variation paths
\[
    \eta_s(o):[0,1]\to\mathbb R^{d_X},
\]
with the uniform topology, such that
\[
    \eta_s(o)(0)=0,
    \qquad
    S^N_{0,1}(\eta_s(o))=\exp(\iota_s(o)).
\]
We assume that \(\eta_s(o)\) uses only the first \(2d_{\mathrm{disc}}\) coordinates.
For event \(i\), write
\[
    \Delta_i=\iota_{s_i}(o_i),
    \qquad
    \eta_i=\eta_{s_i}(o_i).
\]

The default construction in the main text is the special case
\[
    \Delta_i
    =
    \sum_{k\in s_i}
    \left[
        \left(
            x_i^k
            -
            x_{\operatorname{prev}(i,k)}^k
        \right)e_k
        +
        e_{d_{\mathrm{disc}}+k}
    \right],
\]
with all higher-order components equal to zero.
In this degree-one case, one may take
\[
    \eta_i(u)=u\,\Delta_i.
\]
On each fixed event-pattern component, replacing raw observed values by these value increments is a homeomorphism.
The inverse is given by cumulative summation along each channel, starting from the fixed base point.

We now construct the path realisation.
Fix
\[
    (C,x)\in\mathcal D,
    \qquad
    x=((t_0,s_0,o_0),\ldots,(t_{m-1},s_{m-1},o_{m-1})).
\]
Let \(\widetilde C\) denote \(C\) embedded in the final \(d_{\mathrm{cont}}\) coordinates of \(\mathbb R^{d_X}\).

First define the initial auxiliary segment.
For \(u\in I_m\), set
\[
    \rho_\ast(u)
    =
    \min\{1,\max\{0,u\}\}.
\]
This segment encodes the initial continuous value \(\widetilde C_0\).
Its truncated signature is
\[
    S^N_{0,1}(u\mapsto u\widetilde C_0)
    =
    \exp(\widetilde C_0)
    =
    B_C.
\]

For each event \(i=0,\ldots,m-1\), define
\[
    \rho_i^x(u)
    =
    \min\{1,\max\{0,u-1-t_i-i\}\}.
\]
The function \(\rho_i^x\) is zero before the event segment for event \(i\).
It increases linearly from zero to one during that event segment and remains at one after that event segment.

Define
\[
    q_x(u)
    =
    u-\rho_\ast(u)-\sum_{i=0}^{m-1}\rho_i^x(u).
\]
The function \(q_x\) is used only to write the path explicitly.
Since \(C\) contains physical time as a channel, \(q_x\) becomes the physical-time coordinate of the realised path.

Define
\[
    \Psi(x,C)(u)
    =
    \rho_\ast(u)\widetilde C_0
    +
    \left(
        \widetilde C(q_x(u))-\widetilde C_0
    \right)
    +
    \sum_{i=0}^{m-1}
    \eta_i\bigl(\rho_i^x(u)\bigr).
\]
This is a path in \(\mathbb R^{d_X}\).
At \(u=0\), the path starts at zero.
During the initial auxiliary segment, the path moves from \(0\) to \(\widetilde C_0\) while physical time is held fixed at \(0\).
After the initial auxiliary segment, the continuously observed coordinates follow \(\widetilde C(q_x(u))\).
During event segments, the continuously observed coordinates are held fixed while the event path \(\eta_i\) is traversed.

Since \(C_t^\tau=t\), and since the event paths \(\eta_i\) use only the first \(2d_{\mathrm{disc}}\) coordinates, the coordinate \(2d_{\mathrm{disc}}+\tau\) of \(\Psi(x,C)\) is
\[
    \Psi(x,C)^{2d_{\mathrm{disc}}+\tau}(u)
    =
    q_x(u).
\]
Thus physical time is already one of the channels of \(\Psi(x,C)\).
During the initial auxiliary segment and during event segments, this time channel is constant.
During ordinary parts of the path, this time channel increases at unit speed.

For a realised path \(\Psi=\Psi(x,C)\), define
\[
    a_\Psi(t)
    =
    \inf
    \left\{
        u\in I_m
        :
        \Psi^{2d_{\mathrm{disc}}+\tau}(u)=t
    \right\}.
\]
For \(\beta<T\), define
\[
    b_\Psi(\beta)=a_\Psi(\beta).
\]
For \(\beta=T\), define
\[
    b_\Psi(T)=T+m+1.
\]
This final-endpoint convention includes events at \(T\) in the final interval.

For a physical-time interval \([\alpha,\beta)\), define
\[
    S^N_{\alpha,\beta}(\Psi)
    =
    S^N_{a_\Psi(\alpha),b_\Psi(\beta)}(\Psi).
\]
This is the signature of the ordered auxiliary-time portion of \(\Psi\) whose time channel lies in \([\alpha,\beta)\), with events at \(T\) included when \(\beta=T\).

\begin{proposition}[Formal faithfulness]
\label{prop:formal_faithfulness}
Assume the observation-count coordinates are included, \(C\) contains time as a channel, and \(N\geq1\).
Under the local-factor assumptions above, the map
\[
    (C,x)\mapsto\Psi(x,C)
\]
from \(\mathcal D\) to \(\mathcal P\) is continuous and injective.
Moreover, for every \(0\leq\alpha<\beta\leq T\),
\[
    \log S^N_{\alpha,\beta}(\Psi(x,C))
    =
    \Phi_{\alpha,\beta}(x,C).
\]
\end{proposition}

\begin{proof}
We first prove the interval-signature identity.
Let \([\alpha,\beta)\) be a physical-time interval, with the same final-endpoint convention as in the main text.
Suppose the events assigned to this interval are
\[
    \alpha\leq t_{i_1}<\cdots<t_{i_r}<\beta,
\]
with events at \(T\) included when \(\beta=T\).

If \(\alpha=0\), then the selected auxiliary-time portion begins with the initial segment \(u\mapsto u\widetilde C_0\).
The signature of this initial segment is
\[
    B_C=\exp(\widetilde C_0).
\]
The selected portion is then the ordered concatenation of this initial segment, the continuous pieces, and the event pieces.
By Chen's identity,
\[
    S^N_{0,\beta}(\Psi(x,C))
    =
    B_C
    \otimes
    \Gamma_{0,t_{i_1}}
    \otimes
    \exp(\Delta_{i_1})
    \otimes
    \Gamma_{t_{i_1},t_{i_2}}
    \otimes
    \exp(\Delta_{i_2})
    \otimes
    \cdots
    \otimes
    \exp(\Delta_{i_r})
    \otimes
    \Gamma_{t_{i_r},\beta}.
\]
This is exactly \(G_{0,\beta}(x,C)\).

If \(\alpha>0\), then the selected auxiliary-time portion does not include the initial segment.
It is the ordered concatenation of the continuous pieces and event pieces
\[
    \widetilde C|_{[\alpha,t_{i_1}]},
    \eta_{i_1},
    \widetilde C|_{[t_{i_1},t_{i_2}]},
    \eta_{i_2},
    \ldots,
    \eta_{i_r},
    \widetilde C|_{[t_{i_r},\beta]}.
\]
Chen's identity gives
\[
    S^N_{\alpha,\beta}(\Psi(x,C))
    =
    \Gamma_{\alpha,t_{i_1}}
    \otimes
    \exp(\Delta_{i_1})
    \otimes
    \Gamma_{t_{i_1},t_{i_2}}
    \otimes
    \exp(\Delta_{i_2})
    \otimes
    \cdots
    \otimes
    \exp(\Delta_{i_r})
    \otimes
    \Gamma_{t_{i_r},\beta}.
\]
This is exactly \(G_{\alpha,\beta}(x,C)\).

Therefore, for every \(0\leq\alpha<\beta\leq T\),
\[
    \log S^N_{\alpha,\beta}(\Psi(x,C))
    =
    \log G_{\alpha,\beta}(x,C)
    =
    \Phi_{\alpha,\beta}(x,C).
\]

We now prove continuity.
It is enough to work on a fixed event-pattern component.
On such a component, the number of events and observed sets are fixed.
Let
\[
    (C^{(j)},x^{(j)})\to(C,x)
\]
in this component.
Then
\[
    C^{(j)}\to C
\]
uniformly, and for every event \(i\),
\[
    t_i^{(j)}\to t_i,
    \qquad
    o_i^{(j)}\to o_i.
\]
The map
\[
    a\mapsto \min\{1,\max\{0,u-a\}\}
\]
is \(1\)-Lipschitz uniformly in \(u\).
Hence
\[
    \|\rho_i^{x^{(j)}}-\rho_i^x\|_\infty
    \leq
    |t_i^{(j)}-t_i|
    \to0.
\]
It follows that
\[
    \|q_{x^{(j)}}-q_x\|_\infty
    \leq
    \sum_{i=0}^{m-1}
    \|\rho_i^{x^{(j)}}-\rho_i^x\|_\infty
    \to0.
\]

We control the initial continuous segment first.
Since
\[
    C^{(j)}_0\to C_0,
\]
we have
\[
    \sup_{u\in I_m}
    \left\|
        \rho_\ast(u)\widetilde C^{(j)}_0
        -
        \rho_\ast(u)\widetilde C_0
    \right\|
    \to0.
\]

We next control the continuously observed part after subtracting its initial value.
Define
\[
    F_C(t)=\widetilde C(t)-\widetilde C_0.
\]
Then
\[
    F_{C^{(j)}}\to F_C
\]
uniformly.
Since \(F_C\) is continuous on the compact interval \([0,T]\), it is uniformly continuous.
Therefore
\[
\begin{aligned}
    &\sup_{u\in I_m}
    \left\|
        F_{C^{(j)}}(q_{x^{(j)}}(u))
        -
        F_C(q_x(u))
    \right\| \\
    &\qquad\leq
    \|F_{C^{(j)}}-F_C\|_\infty
    +
    \sup_{u\in I_m}
    \left\|
        F_C(q_{x^{(j)}}(u))
        -
        F_C(q_x(u))
    \right\|.
\end{aligned}
\]
The first term tends to zero by uniform convergence.
The second term tends to zero by uniform continuity of \(F_C\) and the uniform convergence of \(q_{x^{(j)}}\) to \(q_x\).

The local realisations depend continuously on the local records.
Thus
\[
    \eta_i^{(j)}\to\eta_i
\]
uniformly on \([0,1]\).
Since each \(\eta_i\) is uniformly continuous on \([0,1]\), we have
\[
    \eta_i^{(j)}(\rho_i^{x^{(j)}}(\cdot))
    \to
    \eta_i(\rho_i^x(\cdot))
\]
uniformly.
Combining the initial segment, the continuous part, and the event part gives
\[
    \|\Psi(x^{(j)},C^{(j)})-\Psi(x,C)\|_\infty\to0.
\]
Since the argument holds on each fixed event-pattern component, and since both the stream space and the target path space use disjoint-union topologies, the map
\[
    (C,x)\mapsto\Psi(x,C)
\]
is continuous.

We now prove injectivity.
The target component of \(\mathcal P\) determines the number of events \(m\).
The time coordinate
\[
    u\mapsto
    \Psi(x,C)^{2d_{\mathrm{disc}}+\tau}(u)
\]
is the function \(q_x(u)\).
Let
\[
    d_x(u)=u-q_x(u).
\]
The initial auxiliary segment is the part where \(0<d_x(u)<1\).
For event \(i\), the event segment is the closure of the set
\[
    \{u\in I_m:1+i<d_x(u)<2+i\}.
\]
On this event segment, the time coordinate is constant and equal to \(t_i\).
Thus the realised path recovers all event times.

The final \(d_{\mathrm{cont}}\) coordinates of \(\Psi(x,C)\) recover \(C\).
Indeed, after the initial auxiliary segment, the final \(d_{\mathrm{cont}}\) coordinates are exactly
\[
    C(q_x(u)).
\]
The event paths \(\eta_i\) use only the first \(2d_{\mathrm{disc}}\) coordinates.
The time coordinate \(q_x\) is surjective onto \([0,T]\).
Therefore, for every \(t\in[0,T]\), choosing any \(u\geq1\) with \(q_x(u)=t\) recovers
\[
    C(t)
\]
from the final \(d_{\mathrm{cont}}\) coordinates of \(\Psi(x,C)(u)\).

Now consider the event segment corresponding to event \(i\).
On this segment, the first \(2d_{\mathrm{disc}}\) coordinates of \(\Psi(x,C)\) traverse \(\eta_i\), up to translation by the previously completed event pieces.
Signatures are invariant under translation.
Therefore the signature of the event segment is
\[
    S^N_{0,1}(\eta_i)
    =
    \exp(\Delta_i).
\]
Taking the truncated logarithm recovers
\[
    \Delta_i.
\]
The degree-one count coordinates of \(\Delta_i\) recover the observed set \(s_i\).
Since
\[
    \Delta_i=\iota_{s_i}(o_i)
\]
and \(\iota_{s_i}\) is injective, the local event record \(o_i\) is recovered by applying the inverse of \(\iota_{s_i}\) on its image.

In the default construction, the degree-one value coordinates of \(\Delta_i\) are
\[
    x_i^k-x_{\operatorname{prev}(i,k)}^k,
    \qquad
    k\in s_i.
\]
Together with the recovered event order and observed sets, these increments recover the raw observed values recursively from the fixed base point.
If
\[
    i_1<i_2<\cdots<i_r
\]
are the events at which channel \(k\) is observed, then
\[
    x_{i_\ell}^k
    =
    \sum_{j=1}^{\ell}
    \left(
        x_{i_j}^k
        -
        x_{\operatorname{prev}(i_j,k)}^k
    \right).
\]
Thus the full input \((C,x)\) is uniquely determined by \(\Psi(x,C)\).
The path realisation is injective.
\end{proof}

\subsection{Invariance of the Embedding}\label{app:invariances}

The path realisation in Proposition~\ref{prop:formal_faithfulness} is faithful when count coordinates are included and $C$ contains time as a channel.
We describe two natural relaxations, dropping the count coordinates, and dropping the continuous variables, in terms of the invariances they induce in the corresponding log-signatures over intervals.

First suppose the count coordinates $e_{d_{\mathrm{disc}} + k}$ are dropped from 
the driving space, so that $d_X = d_{\mathrm{disc}} + d_{\mathrm{cont}}$ 
and the discrete increment becomes
\[
\delta_i \;=\; \sum_{k \in s_i} \bigl( x_i^k - x_{\mathrm{prev}(i, k)}^k \bigr) e_k \;\in\; \mathbb{R}^{d_X}.
\]
An observation event at time $t_i$ with $x_i^k = x_{\mathrm{prev}(i, k)}^k$ 
for every $k \in s_i$ contributes $\delta_i = 0$, hence $E_i = \exp(0) = \mathbf{1}$, 
and the chronological tensor product $G_{\alpha, \beta}(x)$ is unaffected by 
the event's presence or absence.

Define the equivalence relation $\sim_{\mathrm{val}}$ on $\mathcal{X}$ by 
$x \sim_{\mathrm{val}} x'$ if and only if $x'$ is obtained from $x$ by a 
finite sequence of insertions or removals of value-preserving observation 
events. Then $G_{\alpha, \beta}(x) = G_{\alpha, \beta}(x')$ for every 
sub-interval $[\alpha, \beta)$ whenever $x \sim_{\mathrm{val}} x'$, so 
$\varphi_\pi$ without count coordinates is an invariant of 
$\mathcal{X} / \sim_{\mathrm{val}}$ rather than of $\mathcal{X}$.
The same holds true for the realised path embedding.

This is benign for streams in which value repeats carry no information. 
For continuous-valued sensor data, the set of streams with exact value 
repeats is measure-zero under any non-degenerate observation distribution, 
and a stream and its value-preserving insertions are typically 
indistinguishable for prediction purposes. For event-based discrete data, 
where value repeats are common and informative, for example a counter 
incremented at irregular times, or a binary state oscillating between two 
values, the equivalence class loses information that downstream tasks 
need, and count coordinates should be retained.

Now suppose the continuous variables are absent, so $\mathbb{R}^{d_X} = \mathbb{R}^{2 d_{\mathrm{disc}}}$ 
(or $\mathbb{R}^{d_{\mathrm{disc}}}$ if count coordinates are also dropped). 
Then there are no gap factors $\Gamma_{u, v}$, and the interval signature 
reduces to a product of event factors,
\[
G_{[\alpha, \beta)}(x) \;=\; \prod_{i \,:\, t_i \in [\alpha, \beta)} E_i.
\]
The product depends only on which events fall in the interval and their 
order, not on the precise event times $t_i$ themselves. An empty interval 
contributes $\mathbf{1}$ regardless of length.

Define the equivalence relation $\sim_{\mathrm{time}}$ on $\mathcal{X}$ 
by $x \sim_{\mathrm{time}} x'$ if and only if $x'$ is obtained from $x$ 
by changing event times $\{t_i\}$ to any other strictly increasing 
sequence in $[0, T]$, with the same observed channel sets and values. 
Then $G_{\alpha, \beta}(x) = G_{\alpha, \beta}(x')$ for every sub-interval 
$[\alpha, \beta)$ that contains the same event indices in $x$ and $x'$. 
For the partitioned encoding $\varphi_\pi$, this means $\varphi_\pi(x) = \varphi_\pi(x')$ 
whenever $x' \sim_{\mathrm{time}} x$ and the event-to-interval assignments 
are preserved. 
For the realised path embedding, this means that the path consists only of auxiliary event segments. 
The gaps between events carry no path increment, so changing the event times only changes the allocation of these auxiliary segments to physical-time intervals; if that allocation is preserved, the interval signatures of the realised path are unchanged.

The timestamps remain in the data stream as metadata, used to specify 
query intervals and to drive downstream continuous-time models, but they 
no longer enter the encoding itself. This is appropriate when the modelling 
task is concerned with the sequence of events and their values, not with 
the precise inter-event timings; for example, in classification problems 
where temporal regularity is irrelevant, or in event-driven control 
problems where only the event ordering matters. When precise timing 
information is needed, physical time should be included as a continuous variable.

\section{Limitations}
\label{app:limitations}

This section clarifies the main assumptions and scope of the proposed embedding and experiments. 

The first limitation concerns the faithful embedding itself. 
Proposition~\ref{prop:formal_faithfulness} assumes that observation-count coordinates are included and that physical time is present as a continuous channel. 
The count coordinates increase the dimension of the driving path, from \(d_{\mathrm{disc}}+d_{\mathrm{cont}}\) to \(2d_{\mathrm{disc}}+d_{\mathrm{cont}}\). 
This increases the size of the truncated tensor and log-signature representations. 
For truncation depth \(N\), the number of tensor coordinates up to depth \(N\) scales as \(\sum_{k=1}^{N} d_X^k = O(d_X^N)\). Although the log-signature dimension is smaller, scaling like \(O(d_X^N/N)\) by the Witt formula \citep{witt1937treue, reizenstein2020iisignature}, this still grows exponentially in \(N\). 
Thus enforcing exact injectivity through auxiliary channels can increase computational cost, especially for high-dimensional streams or high truncation depths. 
As discussed in Appendix~\ref{app:invariances}, omitting count coordinates or physical time is a modelling choice that intentionally introduces invariances. 
This can be appropriate when repeated observations or exact timestamps are not informative, but it can discard information in event-based or timing-sensitive tasks.

The second limitation concerns the interpretation of universality. 
The universality result is a compact-set approximation result and should not be interpreted as a guarantee of sample efficiency, optimisation success, or improved performance on every dataset. 
It shows that a continuous injective embedding does not obstruct approximation by a universal downstream model. 
Furthermore, faithfulness concerns the realised path embedding; interval summaries at a fixed truncation depth and partition need not be injective.
The empirical behaviour still depends on the model class, truncation depth, optimisation procedure, data distribution, and availability of informative observations.

The final limitation concerns empirical scope. 
Our experiments cover synthetic controlled systems and six UEA time-series classification datasets, including robustness to input dropping. 
They do not exhaust all irregular, asynchronous, sparse, or event-based data regimes. 
The Brownian experiment assumes that second-order information, such as L\'evy area, is available or precomputed. 

\section{Further details on experiments}

Single runs for all experiments can be completed on a 24GB NVIDIA RTX 4090 GPU in less than $24$ hours. We use the following publicly available datasets, libraries, and baseline models:

\begin{itemize}
    \item \textbf{UEA Multivariate Time Series Classification Archive}~\cite{bagnall2018ueamultivariatetimeseries}.  
    License: GPL-3.0. \\
    URL: \url{https://www.timeseriesclassification.com/}, \url{https://github.com/time-series-machine-learning/tsml-repo}

    \item \textbf{S5}~\cite{S5}.  
    License: Apache 2.0. \\
    URL: \url{https://github.com/lindermanlab/S5}

    \item \textbf{LRU}~\cite{orvieto2023resurrecting}.  
    License: MIT. \\
    URL: \url{https://github.com/NicolasZucchet/minimal-LRU}

    \item \textbf{S6/Mamba}~\cite{gu2024mamba}.  
    License: Apache 2.0. \\  
    URL: \url{https://github.com/state-spaces/mamba}

    \item \textbf{Log-NCDE}~\cite{Walker2024LogNCDE}. License: CC-BY-4.0 \\
    URL: \url{https://github.com/Benjamin-Walker/log-neural-cdes}

    \item \textbf{JAX}~\cite{jax2018github}.  
    License: Apache 2.0. \\
    URL: \url{https://github.com/google/jax}
    
    \item \textbf{PyTorch}~\cite{paszke2019pytorchimperativestylehighperformance}.  
    License: BSD Style, see here \url{https://github.com/pytorch/pytorch?tab=License-1-ov-file} \\
    URL: \url{https://github.com/pytorch/pytorch}
\end{itemize}

\subsection{Synthetic coupled data} \label{app:sinusoid}

This section gives further experimental details for the synthetic coupled experiment described in Section~\ref{sec:sinusoid}. 
Each sample is generated on the time interval $[0,T]$ with terminal time $T=10$ and observed channels $d=2$. We restate the underlying continuous-time signal for ease of reference
\[
    z(t) = \omega t + \phi, \qquad
    x_i(t) = A_i \sin(z(t) + \delta_i), \qquad i=1,\ldots,d .
\]
For each sample, the shared angular frequency and global phase are drawn as
\[
    \omega \sim \mathcal{U}(0.8,1.6), \qquad
    \phi \sim \mathcal{U}(0,2\pi),
\]
and the channel-specific parameters are sampled independently as
\[
    A_i \sim \mathcal{U}(0.7,1.3), \qquad
    \delta_i \sim \mathcal{U}(0,2\pi).
\]
The shared latent phase $z(t)$ induces dependence between the two channels, while the channel-specific amplitudes and phase offsets prevent the task from reducing to identical copies of a one-dimensional signal. No observation noise is added in this experiment.

The observations are taken in $(0,T)$ and the target values are computed directly from the underlying continuous dynamics rather than from the observed samples. 
We evaluate four observation regimes. 
In the \textit{synchronous regular} regime, both channels are observed at the same fixed regular grid of 128 interior time points. In the \textit{synchronous irregular} regime, 128 observation times are sampled independently and uniformly from $(0,T)$, then sorted, and shared by both channels.
In the \textit{asynchronous irregular} regime, each channel has its own independently sampled irregular grid of 128 observation times, generated independently and uniformly from $(0,T)$. Therefore, the synchronous regular, synchronous irregular, and asynchronous irregular regimes have the same number of observations per channel. They differ only in the regularity and synchronisation of their observation grids. Finally, in the \textit{asynchronous sparse} regime,
channels are observed asynchronously but at substantially different frequencies. The first channel is dense and second channel is sparse. Dense channels use Poisson sampling with rates $\lambda_i \sim \mathcal{U}(12,20)$, while sparse channels use rates $\lambda_i \sim \mathcal{U}(2,4)$. 
For the asynchronous sparse regime, if a channel has fewer than two observations, then two observation times are sampled uniformly from $(0,T)$ as a fallback. This guarantees that every channel is observed at least twice.

Figures~\ref{fig:example_sync_regular}--\ref{fig:example_async_sparse} show samples from the four regimes. The synchronous regular case (Figure~\ref{fig:example_sync_regular}) has a fixed common input grid, the synchronous irregular case (Figure~\ref{fig:example_sync_irregular}) has an irregular but
shared input grid, the asynchronous irregular case (Figure~\ref{fig:example_async_irregular}) has independent irregular grids, and the asynchronous sparse case (Figure~\ref{fig:example_async_sparse}) combines asynchronous observation with a substantial difference in observation frequency across the channels.

\begin{figure}
    \centering
    \begin{subfigure}[t]{\textwidth}
        \centering
        \includegraphics[width=0.8\linewidth]{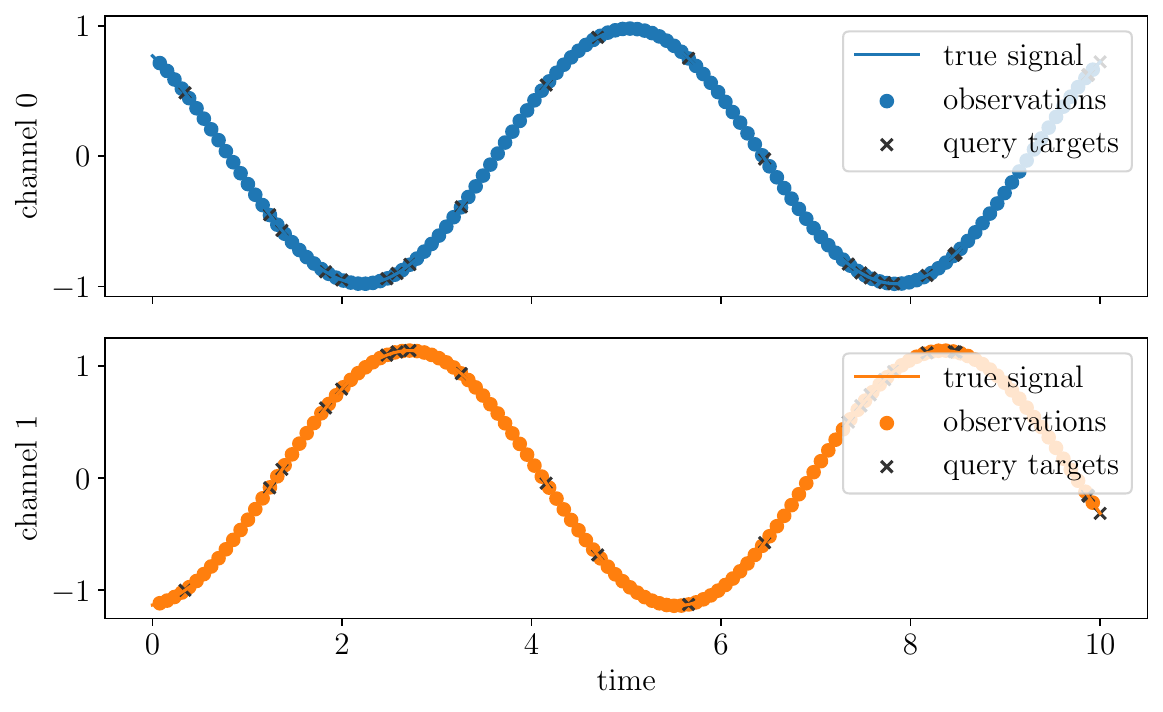}
    \caption{True underlying signal, observation, and query targets.}
    \end{subfigure}
    \begin{subfigure}[t]{\textwidth}
        \centering
        \includegraphics[width=0.8\linewidth]{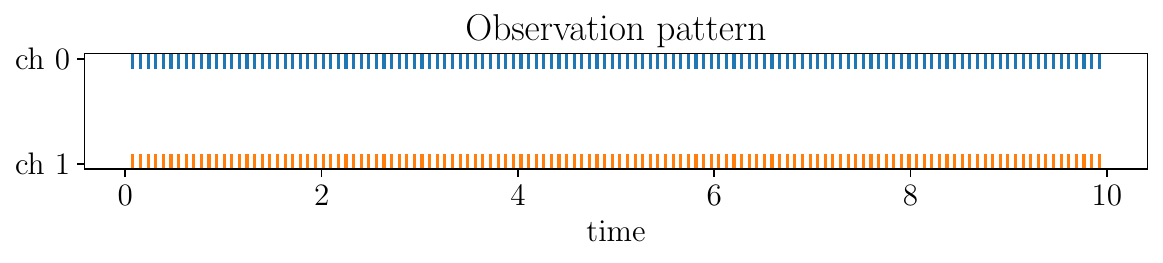}
    \caption{Observation pattern of the two channels.}
    \end{subfigure}
    \caption{\textbf{Observation patterns and signal example for synchronous regular regime.}}
    \label{fig:example_sync_regular}
\end{figure}

\begin{figure}
    \centering
    \begin{subfigure}[t]{\textwidth}
    \centering
        \includegraphics[width=0.8\linewidth]{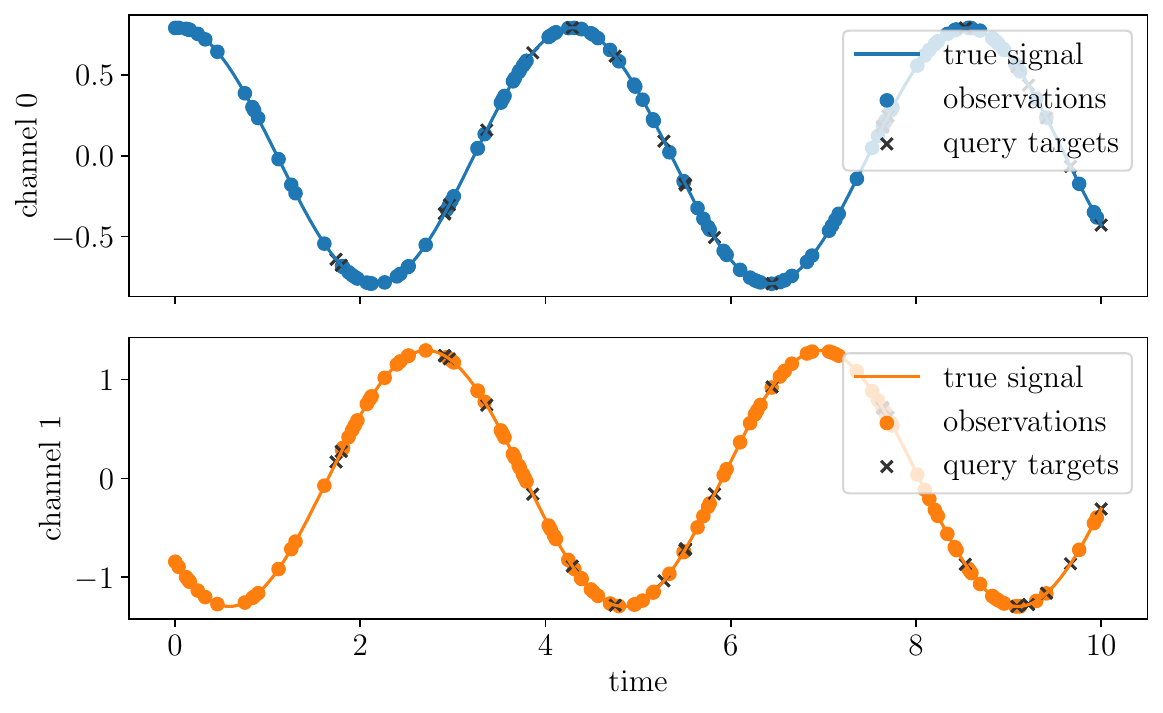}
    \caption{True underlying signal, observation, and query targets.}
    \end{subfigure}
    \begin{subfigure}[t]{\textwidth}
        \centering
        \includegraphics[width=0.8\linewidth]{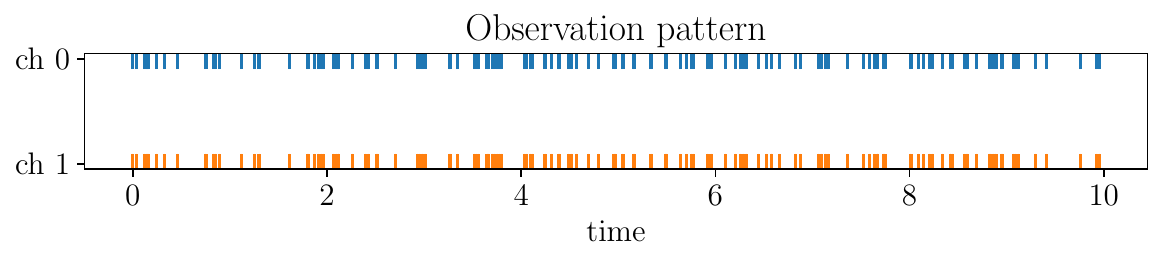}
    \caption{Observation pattern of the two channels.}
    \end{subfigure}
    \caption{\textbf{Observation patterns and signal example for synchronous irregular regime.}}
    \label{fig:example_sync_irregular}
\end{figure}

\begin{figure}
    \centering
    \begin{subfigure}[t]{\textwidth}
        \centering
        \includegraphics[width=0.8\linewidth]{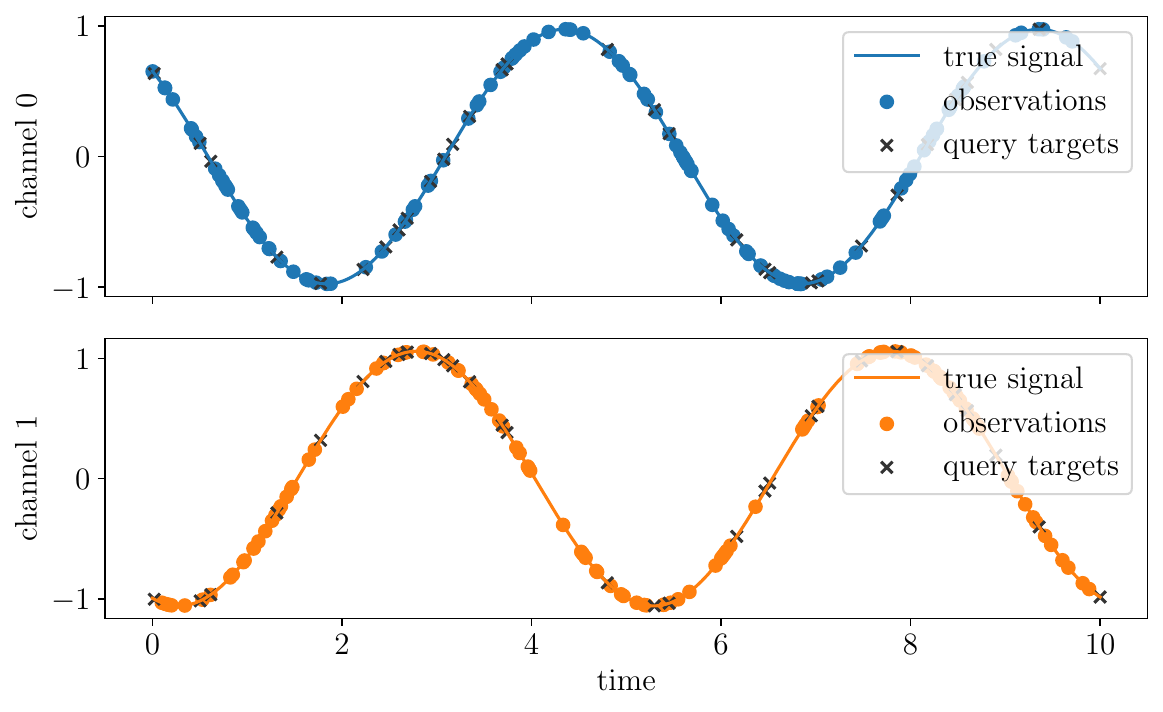}
    \caption{True underlying signal, observation, and query targets.}
    \end{subfigure}
    \begin{subfigure}[t]{\textwidth}
        \centering
        \includegraphics[width=0.8\linewidth]{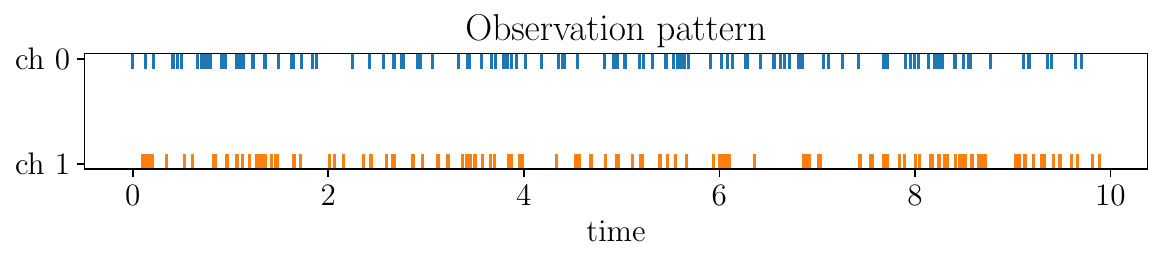}
    \caption{Observation pattern of the two channels.}
    \end{subfigure}
    \caption{\textbf{Observation patterns and signal example for asynchronous irregular regime.}}
    \label{fig:example_async_irregular}
\end{figure}

\begin{figure}
    \centering
    \begin{subfigure}[t]{\textwidth}
        \centering
        \includegraphics[width=0.8\linewidth]{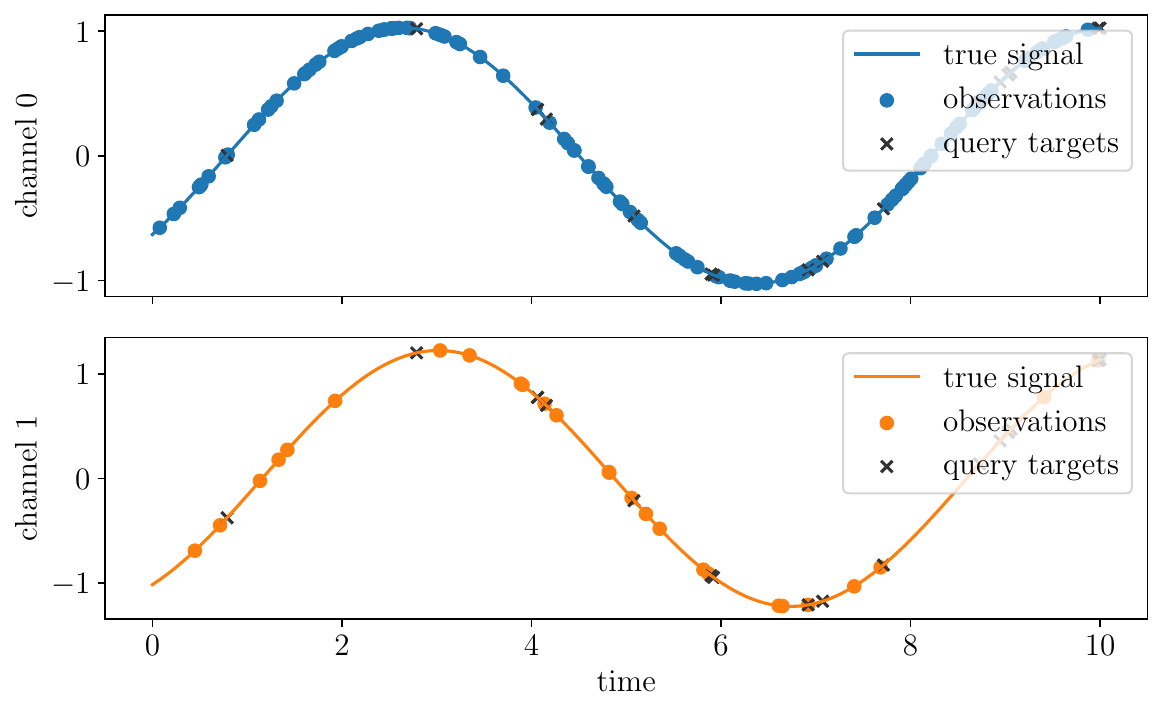}
    \caption{True underlying signal, observation, and query targets.}
    \end{subfigure}
    \begin{subfigure}[t]{\textwidth}
        \centering
        \includegraphics[width=0.8\linewidth]{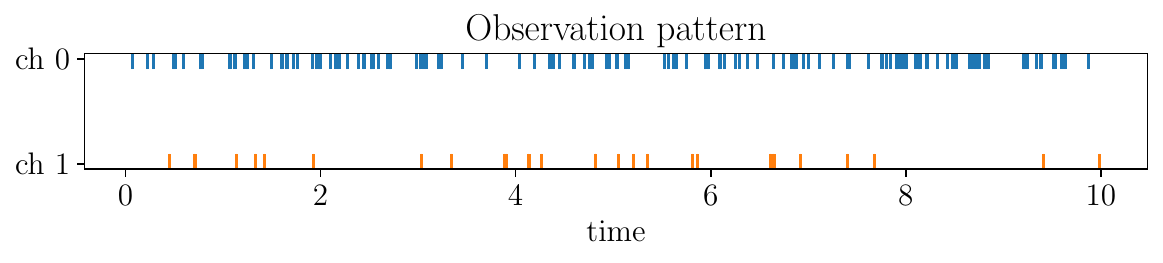}
    \caption{Observation pattern of the two channels.}
    \end{subfigure}
    \caption{\textbf{Observation patterns and signal example for asynchronous sparse regime.}}
    \label{fig:example_async_sparse}
\end{figure}

We convert the observation stream to the increment representation. In this experiment we do not include observation-count coordinates in the model input, since the sampled sinusoid observations almost surely do not contain exact repeated event values. Physical time is appended as a continuously observed channel. Thus the model input has $d+1=3$ channels: two discretely observed value-increment channels and one continuously observed time channel.

For each sequence, we also store the first observed value in each channel. This vector is passed to the input-dependent initialisation of the controlled differential equation after
a $\tanh$ transformation. This gives the model access to the starting observed level of each channel while the subsequent path records increments from one observation to the next.

For every sample, we draw an output query partition independently of the observation times with the number of query intervals $m$ sampled uniformly from $\{16, 17,\ldots,32\}$. We then sample $m-1$ interior query points uniformly from $(0,T)$, sort them, and append the endpoints $0$ and $T$, giving
\[
    0 = q_0 < q_1 < \cdots < q_m = T .
\]
The model receives the log-signature of the observed stream over each interval
$[q_k,q_{k+1})$ and is trained to predict the true endpoint value
\[
    y_k = x(q_{k+1}) \in \mathbb{R}^2 .
\]
Since $q_{k+1}$ is not required to be an observation time, the task is therefore not next-observation prediction. It tests whether the interval summaries allow the model to predict the underlying state on an output grid chosen independently of the input sampling grid.

We train one model on data generated from each observation regime. 
Each training set contains $2048$ samples and each test set contains $512$ samples. We use batch size $32$ and train for $100$ epochs. 
The model is a block diagonal log-SLiCE with hidden dimension $64$, log-signature depth $2$, block size $8$. A linear decoder is applied to the hidden state at each query interval to produce the two-dimensional prediction. 
We optimise the mean squared error over all non-padded query intervals using Adam with learning rate $10^{-3}$ and global gradient clipping at norm $1$.

Because the number of query intervals varies between samples, interval targets are padded within each batch. The loss is computed only on valid intervals. If $\hat{y}_{b,k}$ and
$y_{b,k}$ denote the prediction and target for sample $b$ and interval $k$, and $m_{b,k}\in\{0,1\}$ is the padding mask, the training objective is
\[
    \mathcal{L}
    =
    \frac{
        \sum_{b,k} m_{b,k}
        \left\|\hat{y}_{b,k} - y_{b,k}\right\|_2^2
    }{
        d_{\mathrm{out}} \sum_{b,k} m_{b,k}
    } ,
\]
where $d_{\mathrm{out}}=2$ is the number of predicted channels.

\begin{figure}
    \centering
    \includegraphics[width=\linewidth]{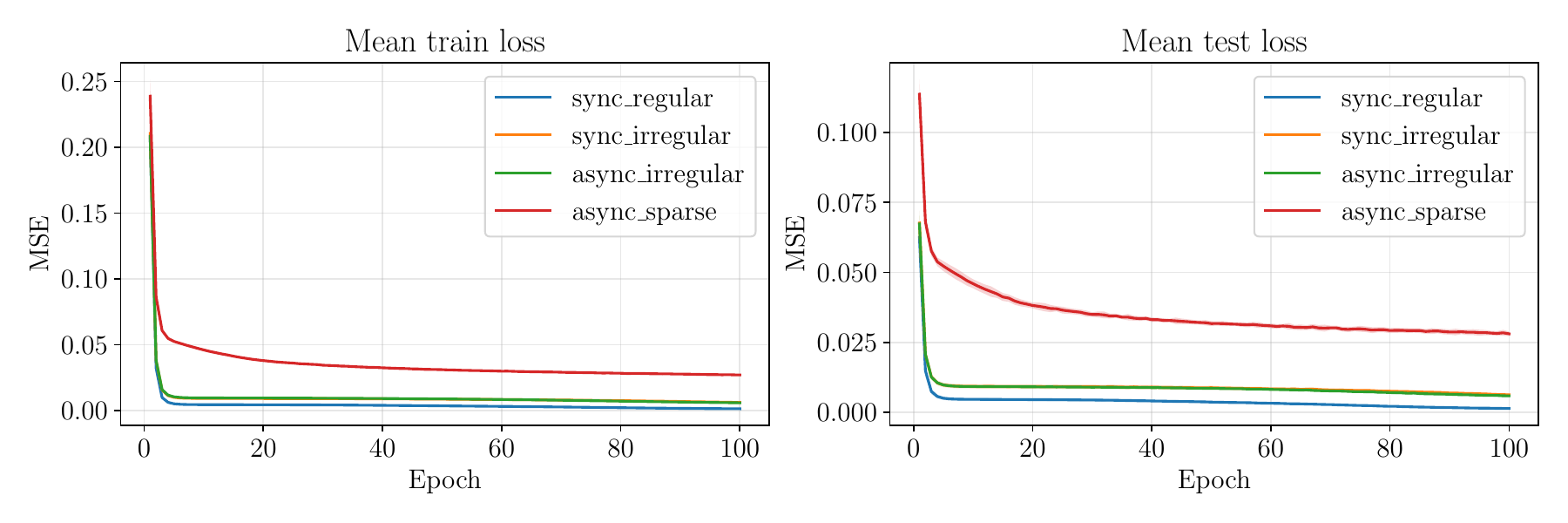}
    \caption{\textbf{Training/test loss over epochs for the synthetic coupled task.}}
    \label{fig:sinusoid_loss_traj}
\end{figure}

After training one model on each regime, we evaluate each trained model on test sets from all four regimes. This gives the cross-regime matrix seen in Figure~\ref{fig:sinusoid}, which reports the mean and standard deviation of the test MSE across five repeats, scaled by $10^{-2}$.

This cross-regime protocol is designed to test sampling-regime robustness. We fixed the observation budget to be 128 observations per channel for the synchronous regular, synchronous irregular, and asynchronous irregular regimes, which allows us to isolate the effects of sampling irregularity and asynchrony from differences in observation frequency. A model trained only on synchronous regular data can exploit the fixed sampling structure seen during training, and therefore achieves very low error on the matching test regime. However, this does not transfer well to irregular,
asynchronous, or sparse observations. In contrast, the asynchronous sparse regime is the hardest training distribution because it exposes the model to asynchronous observations with unequal inter-observation gaps, and severe imbalance in observation frequency. 
Models trained in this regime therefore generalise more robustly to the other regimes, as shown by the more consistent off-diagonal errors in the final row of Figure~\ref{fig:sinusoid}. The similarity between the synchronous irregular and asynchronous irregular rows is consistent with the training
curves in Figure~\ref{fig:sinusoid_loss_traj}. Since these two regimes have the same observation budget per channel, the results suggest that making the two channels asynchronous
does not substantially change the difficulty of the task for our model if the data are already irregular.

In Figures~\ref{fig:predictions_sync_regular}--\ref{fig:predictions_async_sparse}, we see example outputs from trained models under different data sampling regimes.
The predictions closely track the sinusoidal targets in most cases. 
In the asynchronous sparse case, note that the second channel is the sparsely observed channel, hence the larger error for Figure~\ref{fig:predictions_async_sparse}.

\begin{figure}
    \centering
    \includegraphics[width=0.9\linewidth]{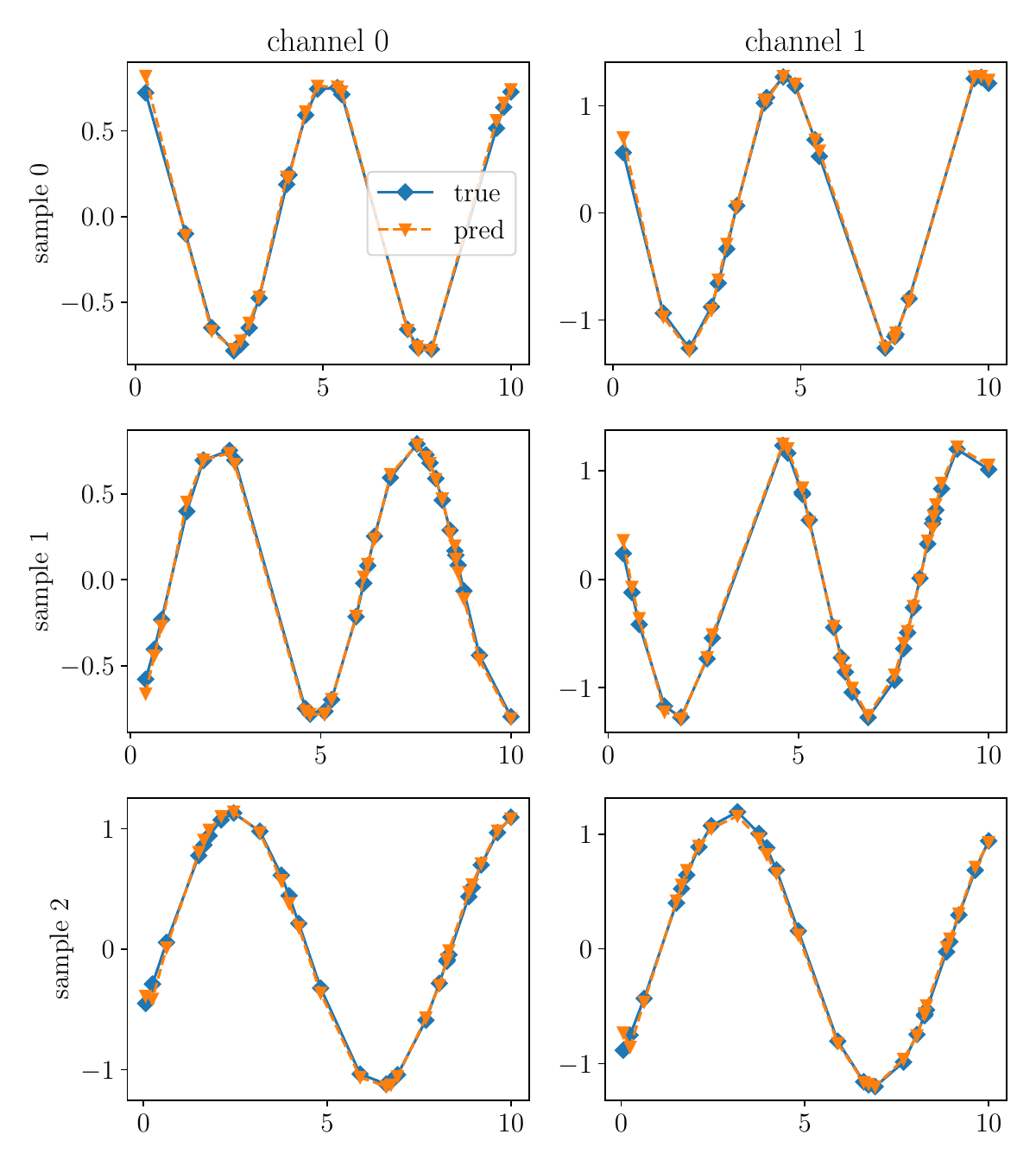}
    \caption{\textbf{Prediction examples for model trained on synchronous regular data.}}
    \label{fig:predictions_sync_regular}
\end{figure}

\begin{figure}
    \centering
    \includegraphics[width=0.9\linewidth]{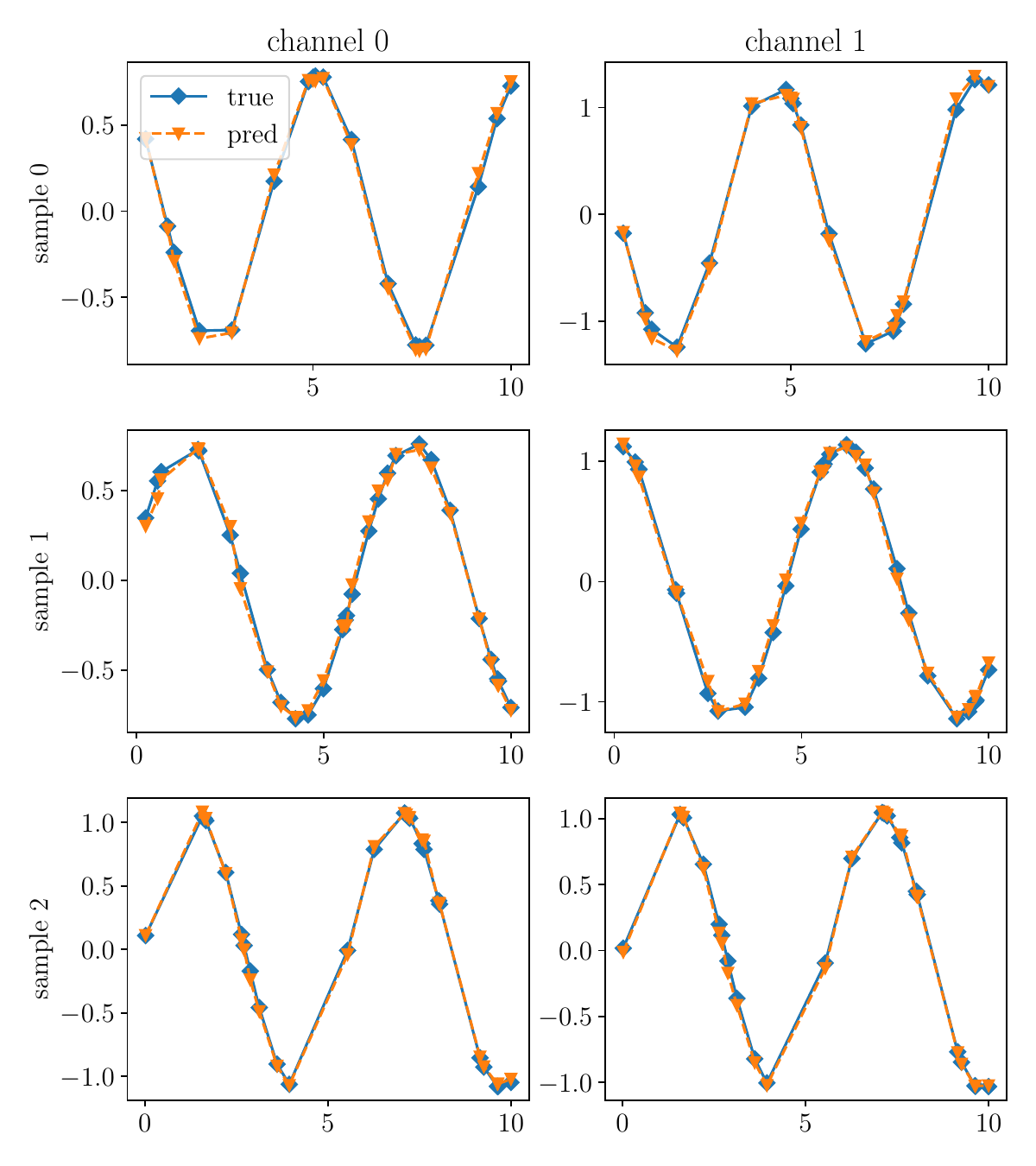}
    \caption{\textbf{Prediction examples for model trained on synchronous irregular data.}}
    \label{fig:predictions_sync_irregular}
\end{figure}

\begin{figure}
    \centering
    \includegraphics[width=0.9\linewidth]{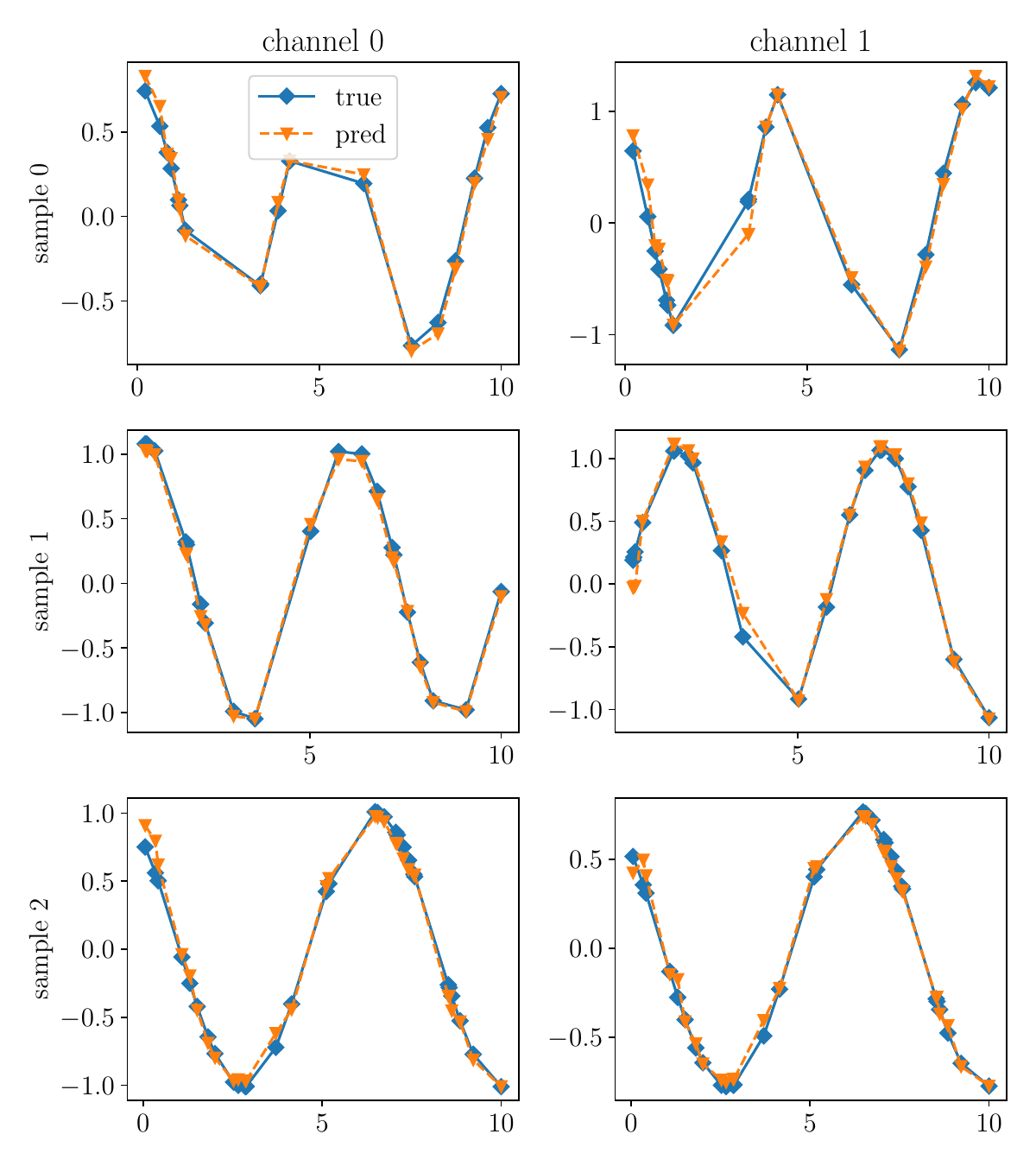}
    \caption{\textbf{Prediction examples for model trained on asynchronous irregular data.}}
    \label{fig:predictions_async_irregular}
\end{figure}

\begin{figure}
    \centering
    \includegraphics[width=0.9\linewidth]{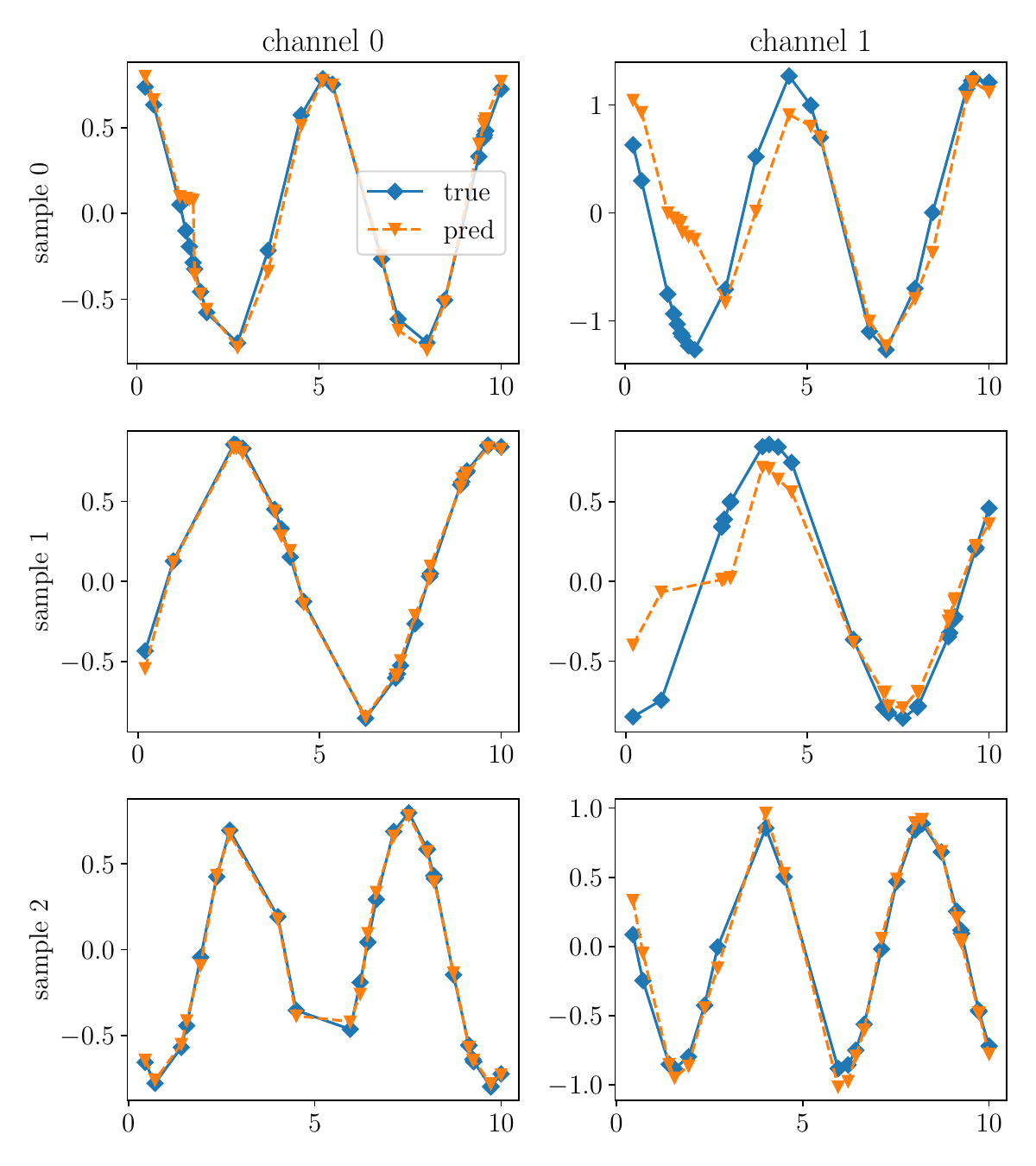}
    \caption{\textbf{Prediction examples for model trained on asynchronous sparse data.}}
    \label{fig:predictions_async_sparse}
\end{figure}

\subsection{Linear system driven by Brownian motion}\label{app:brownian-rde}

This section gives the full experimental details for the Brownian controlled system in Section~\ref{sec:bm_results}. This experiment tests whether our Log-SLiCEs can exploit higher-order information when it is present in the data. The solution of the system considered here depends on second-order Lie bracket information.

We use a precomputed ensemble of $4$-dimensional Brownian rough paths on the interval $[0,1]$. Each path is stored on a fine regular grid with $2048$ time steps. The stored input at each step is a truncated log-signature increment up to depth 2, and so this is of dimension $10$.
That is, each Brownian sample is represented by an array of shape $2048 \times 10$. 

For each Brownian path, we construct targets by simulating a two-dimensional linear controlled system driven by the first two Brownian coordinates, with initial state
\(
    X_0 = (1,0)^\top \in \mathbb{R}^2 .
\)
\[
    \mathrm{d}X_t
    =
    V_1X_t\circ\mathrm{d}W^1_t
    +
    V_2X_t\circ\mathrm{d}W^2_t,
\]

The two degree-one vector fields are fixed linear maps
\[
    V_1
    =
    0.15
    \begin{pmatrix}
        -0.5 & -1.0 \\
         1.0 & -0.5
    \end{pmatrix},
    \qquad
    V_2
    =
    0.15
    \begin{pmatrix}
        -0.2 & 0.8 \\
         0.3 & -0.7
    \end{pmatrix}.
\]
The use of non-commuting matrices here ensures that second-order information cannot be removed without changing the solution.

Although these Brownian streams are four-dimensional, our target system here is driven only by the first two Brownian coordinates. The remaining coordinates are therefore irrelevant for the target dynamics and must be ignored by the model.

We loop over a number of query intervals
\[
    m \in \{2,4,8,16,32,64\}.
\]
The interval endpoints are aligned with the fine Brownian grid. For a given random seed, we sample $m-1$ cut indices uniformly without replacement from the interior grid indices, sort them, and append the endpoints $0$ and $2048$. After rescaling by $2048$, this gives a random partition
\[
    0 = q_0 < q_1 < \cdots < q_m = 1 .
\]
The same partition is used for all samples within a run. The target associated with interval $[q_k,q_{k+1})$ is the simulated state at the right endpoint,
\[
    y_k = X(q_{k+1}) \in \mathbb{R}^2 .
\]
Although the model predicts endpoint values on a query grid that is separate from the fine input grid, the query endpoints are chosen to coincide with fine-grid points to make use of the available second level information.

When only depth-1 information is used, we keep only the first four coordinates, i.e. the Brownian increments themselves. In both cases, physical time is appended as a continuously observed channel. Therefore the model input has $10+1=11$ channels in the depth-2 setting and $4+1=5$ channels in the depth-1 setting.

We use block diagonal Log-SLiCE with hidden dimension 64 and block size 8. A linear decoder maps the hidden state at each query interval to the two-dimensional predicted state. Unlike the sinusoid experiment, we do not use input-dependent initialisation in this experiment.

Each run uses $2048$ training paths and $512$ test paths. We train with batch size $32$ for $40$ epochs using Adam with learning rate $10^{-3}$ and global gradient clipping at norm $1$.
For each value of $m \in \{2,4,8,16,32,64\}$, we repeat the experiment over five random seeds. Like the sinusoid task, the training objective is the masked mean squared error over non-padded intervals.

\begin{table}[t]
\centering
\caption{
\textbf{Test MSE versus number of intervals.}
All values are scaled by $10^{-5}$ and reported as mean $\pm$ standard deviation across 5 seeds.
}
\begin{tabular}{ccc}
\toprule
Intervals & Level 1 & Level 2 \\
\midrule
2  & $6.97 \pm 0.93$ & $0.29 \pm 0.03$ \\
4  & $3.72 \pm 0.86$ & $0.20 \pm 0.06$ \\
8  & $2.10 \pm 0.46$ & $0.15 \pm 0.03$ \\
16 & $0.72 \pm 0.09$ & $0.12 \pm 0.01$ \\
32 & $0.41 \pm 0.04$ & $0.11 \pm 0.01$ \\
64 & $0.28 \pm 0.02$ & $0.11 \pm 0.01$ \\
\bottomrule
\end{tabular}
\label{tab:intervals_vs_error}
\end{table}

Figure~\ref{fig:bm_log} and Table~\ref{tab:intervals_vs_error} report the test MSE as a function of the number of query intervals. Example output paths can be seen in Figures~\ref{fig:bm_results} and \ref{fig:bm_results2}. As expected, the level-2 model achieves substantially lower error, especially
when the number of query intervals is small.
As the number of intervals increases, the gap between the two models decreases. When the partition is refined, each interval becomes shorter, and the effect of unobserved second-level terms on each interval becomes smaller. 
The level-1 model can therefore partially compensate for the lack of second level information by using more frequent updates. 
However, the level-2 model remains more accurate across the interval counts considered, confirming that our approach can incorporate and exploit higher-order information directly when such information is available.

\begin{figure}[t]
    \centering
    \includegraphics[width=0.75\linewidth]{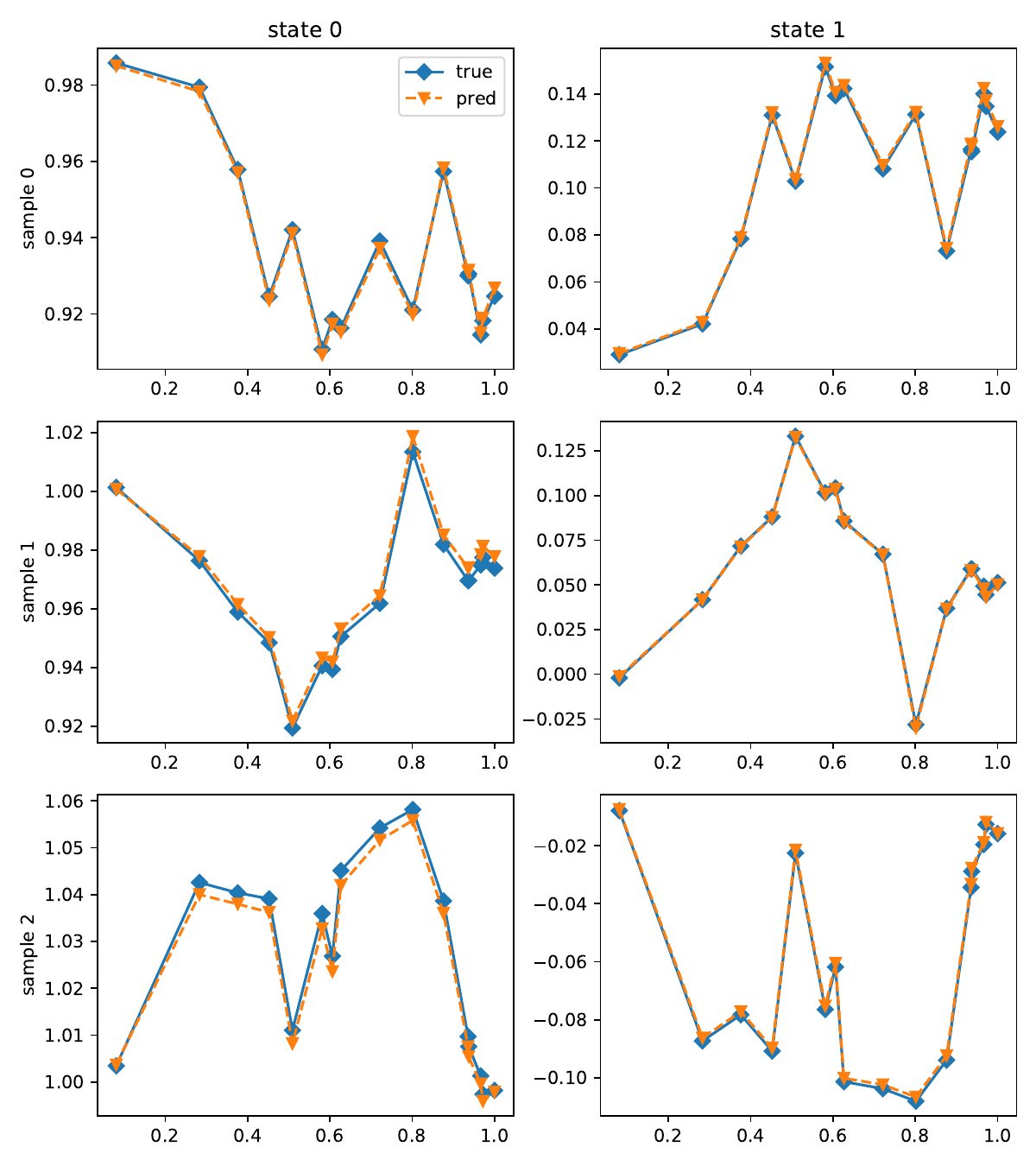}
    \caption{\textbf{Prediction of a linear system driven by Brownian motion using a model trained with only increments (level 1) data. }
    Three examples are taken from the test set with $16$ intervals.}
    \label{fig:bm_results}
\end{figure}

\begin{figure}[t]
    \centering
    \includegraphics[width=0.75\linewidth]{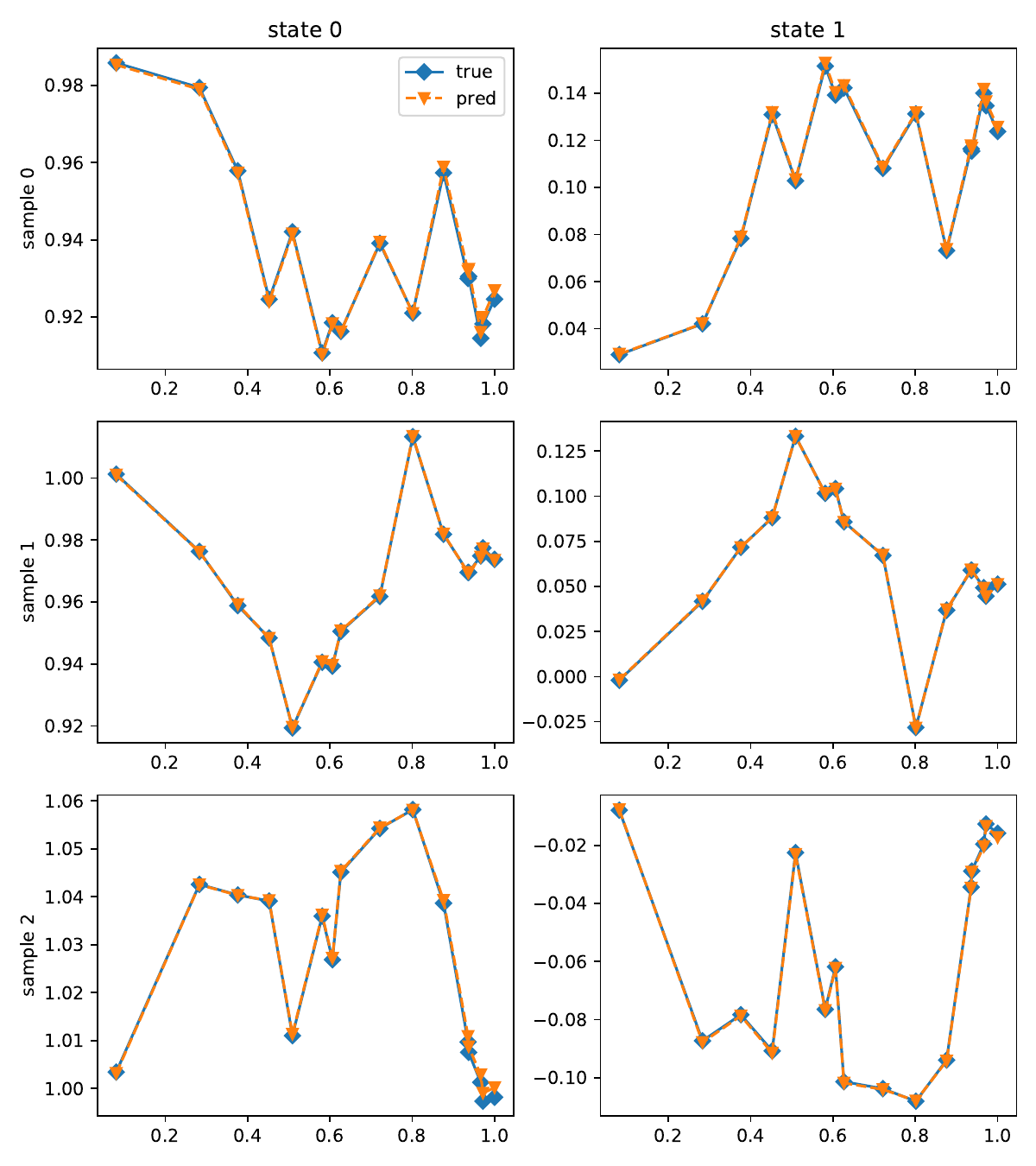}
    \caption{\textbf{Prediction of a linear system driven by Brownian motion using a model trained with both increments and area (level 2) data. }
    Three examples are taken from the test set with $16$ intervals.}
    \label{fig:bm_results2}
\end{figure}

\subsection{UEA experiments}
\label{app:uea-experiments}

We follow the UEA multivariate time series classification archive (UEA-MTSCA) \citep{bagnall2018ueamultivariatetimeseries} protocol of \citet{Walker2024LogNCDE} and the SLiCE protocol of \citet{walker2025structuredlinearcdesmaximally}.
The benchmark consists of EigenWorms, EthanolConcentration, Heartbeat, MotorImagery, SelfRegulationSCP1, and SelfRegulationSCP2.
As in \citet{Walker2024LogNCDE}, duplicated EigenWorms series are removed before splitting.
The original train and test cases are pooled and re-split into train, validation, and test sets with relative sizes $70/15/15$.
Hyperparameters are selected using validation accuracy on a fixed random split and are then kept fixed for evaluation over five random seeds.

The recurrent baselines are S5, LRU, S6, and Mamba \citep{S5, orvieto2023resurrecting, gu2024mamba}.
The continuous baselines are an NCDE with Hermite cubic interpolation and backward differences \citep{kidger2020neuralcde}, and a Log-NCDE, diagonal SLiCE (D-SLiCE), and block-diagonal SLiCE (BD-SLiCE) using our embedding \citep{Walker2024LogNCDE, walker2025structuredlinearcdesmaximally}.
The aim of these experiments is not to establish a new state of the art on the UEA-MTSCA benchmark.
The selected baselines are chosen to match the protocols of \citet{Walker2024LogNCDE} and \citet{walker2025structuredlinearcdesmaximally}, and we keep the selected hyperparameters fixed rather than re-tuning each model for the dropped-input setting.
Recent models, including Rough Transformers \citep{moreno2024rough}, oscillatory state-space models such as LinOSS \citep{rusch2025oscillatory} and D-LinOSS \citep{boyer2025learning}, and nonlinear parallel state-space models such as LrcSSM \citep{farsang2025parallelization}, report higher no-drop accuracies on this six-dataset benchmark.
Our comparison instead focuses on the effect of the proposed embedding and interval-level Log-ODE computation under controlled hyperparameters and increasing input sparsification.

Since the task is classification, time is treated as a discretely observed variable when included in the driving path, and there are no continuously observed variables.
For D-SLiCE and BD-SLiCE, we use the Log-NCDE hyperparameters selected by \citet{Walker2024LogNCDE}, replacing only the non-linear vector field by the corresponding structured linear vector field, following \citet{walker2025structuredlinearcdesmaximally}.
For BD-SLiCE, the block size is $b_i=4$.
For the SLiCE models with the Log-ODE method, flows are composed using an associative parallel scan with chunks of size $128$.
Table~\ref{tab:uea_logncde_hparams} gives the Log-NCDE hyperparameters used by Log-NCDE, D-SLiCE, and BD-SLiCE.
Table~\ref{tab:uea_recurrent_hparams} gives the selected hyperparameters for the recurrent baselines.
All of these hyperparameters are kept fixed across the no-drop and input-dropping experiments.

All reported accuracies are test accuracies in percent, reported as mean $\pm$ standard deviation over five seeds.
Average accuracy is the macro-average over the six datasets.
Average rank is computed by ranking the models separately on each dataset at a fixed dropping level and averaging the ranks over datasets.
No-drop values marked with $^\dagger$ are taken from \citet{Walker2024LogNCDE}, and no-drop values marked with $^\ddagger$ are taken from \citet{walker2025structuredlinearcdesmaximally}.
All dropped-input results are produced by our runs using the same selected hyperparameters, without retuning for the drop level.

\begin{table}[t]
\centering
\small
\setlength{\tabcolsep}{5pt}
\renewcommand{\arraystretch}{1.15}
\caption{\textbf{Log-NCDE hyperparameters for the UEA path models.}
D-SLiCE and BD-SLiCE use the same settings as Log-NCDE, replacing only the vector field structure.}
\label{tab:uea_logncde_hparams}
\begin{tabular}{lcccccc}
\toprule
\multirow{2}{*}{Dataset}
& \multirow{2}{*}{LR}
& \multirow{2}{*}{Include time}
& \multirow{2}{*}{Hidden Dimension}
& \multicolumn{2}{c}{Log-ODE}
& \multirow{2}{*}{$\lambda$} \\
\cmidrule(lr){5-6}
& & & & Depth & Step & \\
\midrule
EW   & $10^{-3}$ & $\checkmark$ & $128$ & $2$ & $12$ & $10^{-3}$ \\
EC   & $10^{-4}$ & $\checkmark$ & $64$  & $1$ & $1$  & $10^{-6}$ \\
HB   & $10^{-3}$ & $\checkmark$ & $16$  & $2$ & $2$  & $10^{-6}$ \\
MI   & $10^{-3}$ & $\times$     & $16$  & $2$ & $16$ & $10^{-3}$ \\
SCP1 & $10^{-4}$ & $\times$     & $64$  & $2$ & $16$ & $0$ \\
SCP2 & $10^{-4}$ & $\times$     & $128$ & $2$ & $4$  & $10^{-3}$ \\
\bottomrule
\end{tabular}
\vspace{0.1cm}
\end{table}

\begin{table*}[t]
\centering
\scriptsize
\setlength{\tabcolsep}{4pt}
\renewcommand{\arraystretch}{1.12}
\caption{\textbf{Selected recurrent-model hyperparameters for the UEA experiments.}
The values are those selected by the hyperparameter optimisation of \citet{Walker2024LogNCDE} and are kept fixed in all dropping experiments.}
\label{tab:uea_recurrent_hparams}
\begin{tabular}{llcccccc}
\toprule
Model & Hyperparameter & EW & EC & HB & MI & SCP1 & SCP2 \\
\midrule
\multirow{5}{*}{LRU}
& LR & $10^{-3}$ & $10^{-5}$ & $10^{-4}$ & $10^{-3}$ & $10^{-3}$ & $10^{-3}$ \\
& Include time & $\times$ & $\checkmark$ & $\checkmark$ & $\times$ & $\times$ & $\checkmark$ \\
& Hidden dimension & $64$ & $64$ & $128$ & $16$ & $64$ & $64$ \\
& Layers & $4$ & $6$ & $2$ & $6$ & $2$ & $2$ \\
& State dimension & $64$ & $16$ & $256$ & $256$ & $16$ & $16$ \\
\midrule
\multirow{6}{*}{S5}
& LR & $10^{-4}$ & $10^{-5}$ & $10^{-3}$ & $10^{-3}$ & $10^{-3}$ & $10^{-4}$ \\
& Include time & $\checkmark$ & $\checkmark$ & $\times$ & $\times$ & $\times$ & $\checkmark$ \\
& Hidden dimension & $64$ & $128$ & $128$ & $16$ & $128$ & $16$ \\
& Layers & $2$ & $2$ & $4$ & $6$ & $6$ & $2$ \\
& State dimension & $16$ & $16$ & $16$ & $64$ & $16$ & $64$ \\
& Init. blocks & $8$ & $8$ & $4$ & $4$ & $8$ & $2$ \\
\midrule
\multirow{5}{*}{S6}
& LR & $10^{-3}$ & $10^{-5}$ & $10^{-3}$ & $10^{-3}$ & $10^{-4}$ & $10^{-3}$ \\
& Include time & $\times$ & $\checkmark$ & $\checkmark$ & $\checkmark$ & $\times$ & $\times$ \\
& Hidden dimension & $16$ & $16$ & $16$ & $16$ & $64$ & $16$ \\
& Layers & $4$ & $4$ & $4$ & $4$ & $2$ & $2$ \\
& State dimension & $64$ & $16$ & $16$ & $256$ & $16$ & $256$ \\
\midrule
\multirow{7}{*}{Mamba}
& LR & $10^{-3}$ & $10^{-3}$ & $10^{-4}$ & $10^{-5}$ & $10^{-5}$ & $10^{-4}$ \\
& Include time & $\checkmark$ & $\checkmark$ & $\times$ & $\times$ & $\times$ & $\checkmark$ \\
& Hidden dimension & $16$ & $64$ & $64$ & $128$ & $128$ & $64$ \\
& Layers & $6$ & $4$ & $4$ & $2$ & $2$ & $6$ \\
& State dimension & $64$ & $256$ & $256$ & $64$ & $16$ & $64$ \\
& Conv. dimension & $2$ & $4$ & $2$ & $3$ & $4$ & $2$ \\
& Expansion factor & $1$ & $4$ & $4$ & $1$ & $1$ & $2$ \\
\bottomrule
\end{tabular}
\end{table*}

\subsubsection{EigenWorms}
\label{app:uea-eigenworms}

EigenWorms is additionally used for the timing and memory comparison in Table~\ref{tab:log_method_comparison_ew} because it is the longest UEA dataset in the benchmark, with $17{,}984$ observations per series.
This makes it a direct stress test of whether the Log-ODE method can reduce the cost of long input streams.
For this experiment, we keep the hyperparameters fixed to the values selected in the UEA protocol above and compare each model family with and without interval-level Log-ODE updates.

For the nonlinear NCDE family, the non-Log-ODE row corresponds to the standard NCDE baseline and the Log-ODE row corresponds to the Log-NCDE.
For the dense linear and block-diagonal linear families, all rows are computed in parallel-in-time using an associative scan.
The non-Log-ODE rows use depth-$1$ Log-ODE updates on intervals aligned with consecutive observations.
This is equivalent to solving the CDE exactly with a linearly interpolated observation path.
The Log-ODE rows instead use the selected coarser interval partition and higher-order interval log-signatures, then compose the resulting linear flows with the same associative-scan mechanism.
Thus the comparison isolates the effect of replacing observation-level updates by interval-level Log-ODE updates, while keeping the architecture family and selected hyperparameters fixed.

Timing and memory are measured on an NVIDIA H100 GPU with batch size $1$.
The reported time is the wall-clock time per training step averaged over $1000$ training steps and the reported memory is peak GPU memory.
The EigenWorms results show that the Log-ODE method is beneficial for all three model families.
In particular, BD-SLiCE with Log-ODE achieves higher accuracy than the standard NCDE baseline while reducing the time per training step from $26.02$ seconds to $0.009$ seconds.

\subsubsection{Input dropping}
\label{app:uea-drop-experiments}

The input-dropping experiments test robustness to sparsification of the observed stream.
For each drop level $p \in \{0.3,0.7,0.95\}$, a random mask is sampled independently for each series and applied to the input before model-specific preprocessing.
All models use the same hyperparameters as in the no-drop setting, with no retuning for the drop level.
The same drop level is used throughout each run, and results are averaged over five seeds.

For path-based models, dropping observations changes the observed stream but should not change the physical time horizon.
We therefore keep the original Log-ODE query partition from the undropped series fixed, and recompute each interval log-signature using the surviving observations whose timestamps fall in that interval.
If no discrete observation remains in an interval, the interval still contributes time, if it is included as a channel in the path.
This preserves the original temporal alignment and avoids shortening or rescaling the path after observations are removed.

For the discrete recurrent models S5, LRU, S6, and Mamba, the randomly retained observations are passed directly to the model in chronological order.
When the selected hyperparameters include time, the timestamp is retained as an input channel after dropping.
When time is not included, the model receives only the retained observed values.

Table~\ref{tab:uea_drop_summary} summarises the results across the six datasets.
Tables~\ref{tab:uea_drop_per_dataset_nodrop}--\ref{tab:uea_drop_per_dataset_drop95} give the corresponding per-dataset results.
The continuous-time models using our embedding remain competitive across all drop levels.
BD-SLiCE is the most stable model under severe sparsification, with its macro-average accuracy changing from $64.0\%$ with no dropping to $62.5\%$ with $95\%$ dropping.

\begin{table*}[t]
  \centering
  \scriptsize
  \setlength{\tabcolsep}{4.5pt}
    \caption{\textbf{Per-dataset test accuracy (mean $\pm$ std over 5 seeds, in \%) for the UEA benchmark at no drop, with macro-average accuracy and average dataset rank.} Lower rank is better.}
  \begin{tabular}{lcccccccc}
  \toprule
  Model & EigenWorms & Ethanol & Heartbeat & MotorImagery & SCP1 & SCP2 & Avg.\ Acc. & Avg.\ Rank \\
  \midrule
  S5 & $81.1 \pm 3.7$ & $24.1 \pm 4.3$ & \underline{$77.7 \pm 5.5$} & $47.7 \pm 5.5$ & $\mathbf{89.9 \pm 4.6}$ & $50.5 \pm 2.6$ & $61.8$ & $4.67$ \\
  LRU & $\mathbf{87.8 \pm 2.8}$ & $21.5 \pm 2.1$ & $\mathbf{78.4 \pm 6.7}$ & $48.4 \pm 5.0$ & $82.6 \pm 3.4$ & $51.2 \pm 3.6$ & $61.7$ & $4.50$ \\
  S6 & $85.0 \pm 16.1$ & $26.4 \pm 6.4$ & $76.5 \pm 8.3$ & $51.3 \pm 4.7$ & $82.8 \pm 2.7$ & $49.8 \pm 9.5$ & $62.0$ & $5.00$ \\
  Mamba & $70.9 \pm 15.8$ & $27.9 \pm 4.5$ & $76.2 \pm 3.8$ & $47.7 \pm 4.5$ & $80.7 \pm 1.4$ & $48.2 \pm 3.9$ & $58.6$ & $6.67$ \\
  NCDE & $75.0 \pm 3.9$ & \underline{$29.9 \pm 6.5$} & $73.9 \pm 2.6$ & $49.5 \pm 2.8$ & $79.8 \pm 5.6$ & $53.0 \pm 2.8$ & $60.2$ & $5.33$ \\
  Log-NCDE & $85.6 \pm 5.1$ & $\mathbf{34.4 \pm 6.4}$ & $75.2 \pm 4.6$ & \underline{$53.7 \pm 5.3$} & $83.1 \pm 2.8$ & \underline{$53.7 \pm 4.1$} & $\mathbf{64.3}$ & \underline{$3.00$} \\
  D-SLiCE & $79.4 \pm 5.7$ & $27.1 \pm 4.6$ & $72.9 \pm 5.0$ & $\mathbf{54.4 \pm 6.3}$ & $83.5 \pm 2.1$ & $53.0 \pm 5.8$ & $61.7$ & $4.50$ \\
  BD-SLiCE & \underline{$86.1 \pm 3.5$} & $28.6 \pm 6.4$ & $77.4 \pm 5.6$ & $53.0 \pm 2.0$ & \underline{$84.9 \pm 1.9$} & $\mathbf{54.0 \pm 7.4}$ & \underline{$64.0$} & $\mathbf{2.33}$ \\
  \bottomrule
  \end{tabular}
  \label{tab:uea_drop_per_dataset_nodrop}
\end{table*}

\begin{table*}[t]
  \centering
  \scriptsize
  \setlength{\tabcolsep}{4.5pt}
  \caption{\textbf{Per-dataset test accuracy (mean $\pm$ std over 5 seeds, in \%) for the UEA benchmark at 30\% drop, with macro-average accuracy and average dataset rank.} Lower rank is better.}
  \begin{tabular}{lcccccccc}
  \toprule
  Model & EigenWorms & Ethanol & Heartbeat & MotorImagery & SCP1 & SCP2 & Avg.\ Acc. & Avg.\ Rank \\
  \midrule
  S5 & $79.4 \pm 4.2$ & $26.6 \pm 1.1$ & \underline{$74.8 \pm 2.4$} & $\mathbf{58.9 \pm 2.9}$ & $\mathbf{88.7 \pm 1.6}$ & $51.9 \pm 4.4$ & $63.4$ & $3.25$ \\
  LRU & $78.9 \pm 6.2$ & $25.3 \pm 4.7$ & $70.3 \pm 4.6$ & $56.8 \pm 5.6$ & $82.6 \pm 3.1$ & $51.2 \pm 1.7$ & $60.9$ & $5.58$ \\
  S6 & $80.0 \pm 10.5$ & $25.9 \pm 6.7$ & $74.0 \pm 7.5$ & $42.4 \pm 4.0$ & $78.0 \pm 5.2$ & $50.1 \pm 8.4$ & $58.4$ & $6.50$ \\
  Mamba & $65.3 \pm 15.8$ & $27.2 \pm 1.8$ & $74.8 \pm 6.2$ & $51.3 \pm 2.5$ & $80.4 \pm 1.9$ & $\mathbf{57.7 \pm 4.6}$ & $59.5$ & $4.83$ \\
  NCDE & $77.8 \pm 4.6$ & $28.9 \pm 7.0$ & $67.1 \pm 4.3$ & $51.2 \pm 5.6$ & $80.5 \pm 6.5$ & $51.9 \pm 3.8$ & $59.6$ & $5.92$ \\
  Log-NCDE & \underline{$82.8 \pm 7.3$} & $\mathbf{37.2 \pm 3.6}$ & $74.2 \pm 4.6$ & $54.7 \pm 6.2$ & $82.6 \pm 2.2$ & $53.7 \pm 6.7$ & $\mathbf{64.2}$ & \underline{$3.08$} \\
  D-SLiCE & $76.7 \pm 5.4$ & $27.3 \pm 4.8$ & $\mathbf{77.4 \pm 6.5}$ & $51.2 \pm 3.9$ & $84.0 \pm 1.8$ & \underline{$53.7 \pm 5.3$} & $61.7$ & $3.83$ \\
  BD-SLiCE & $\mathbf{84.4 \pm 4.2}$ & \underline{$32.7 \pm 4.6$} & $69.0 \pm 10.0$ & \underline{$57.9 \pm 3.1$} & \underline{$85.4 \pm 2.8$} & $53.0 \pm 3.4$ & \underline{$63.7$} & $\mathbf{3.00}$ \\
  \bottomrule
  \end{tabular}
  \label{tab:uea_drop_per_dataset_drop30}
\end{table*}

\begin{table*}[t]
  \centering
  \scriptsize
  \setlength{\tabcolsep}{4.5pt}
    \caption{\textbf{Per-dataset test accuracy (mean $\pm$ std over 5 seeds, in \%) for the UEA benchmark at 70\% drop, with macro-average accuracy and average dataset rank.} Lower rank is better.}
  \begin{tabular}{lcccccccc}
  \toprule
  Model & EigenWorms & Ethanol & Heartbeat & MotorImagery & SCP1 & SCP2 & Avg.\ Acc. & Avg.\ Rank \\
  \midrule
  S5 & $72.8 \pm 3.2$ & $27.3 \pm 1.3$ & $70.3 \pm 4.3$ & $\mathbf{56.5 \pm 6.4}$ & \underline{$85.4 \pm 2.2$} & $48.4 \pm 6.2$ & $60.1$ & $4.83$ \\
  LRU & $81.1 \pm 6.4$ & $25.3 \pm 4.7$ & $75.5 \pm 6.6$ & \underline{$54.4 \pm 4.6$} & $78.4 \pm 3.8$ & $51.2 \pm 5.9$ & $61.0$ & $5.00$ \\
  S6 & \underline{$84.7 \pm 6.3$} & $25.3 \pm 7.4$ & $76.1 \pm 4.1$ & $51.3 \pm 2.2$ & $79.2 \pm 1.9$ & $51.7 \pm 3.0$ & $61.4$ & $4.00$ \\
  Mamba & $69.1 \pm 13.3$ & $27.2 \pm 1.8$ & $\mathbf{76.7 \pm 5.8}$ & $53.0 \pm 2.2$ & $79.2 \pm 4.8$ & $\mathbf{58.4 \pm 3.2}$ & $60.6$ & $4.00$ \\
  NCDE & $78.9 \pm 2.8$ & $24.1 \pm 6.2$ & $69.7 \pm 5.9$ & $50.5 \pm 2.3$ & $80.2 \pm 5.8$ & $49.8 \pm 5.9$ & $58.9$ & $6.42$ \\
  Log-NCDE & $\mathbf{87.2 \pm 5.4}$ & $\mathbf{36.2 \pm 5.9}$ & $73.5 \pm 4.4$ & $50.2 \pm 5.5$ & $78.8 \pm 7.1$ & $50.5 \pm 7.6$ & $\mathbf{62.7}$ & $4.67$ \\
  D-SLiCE & $80.6 \pm 7.2$ & $29.4 \pm 8.2$ & \underline{$76.5 \pm 5.3$} & $50.2 \pm 4.1$ & $84.5 \pm 3.7$ & $51.9 \pm 2.4$ & $62.2$ & \underline{$3.83$} \\
  BD-SLiCE & $81.7 \pm 6.2$ & \underline{$31.9 \pm 5.1$} & $72.6 \pm 5.8$ & $50.5 \pm 3.2$ & $\mathbf{85.9 \pm 3.0}$ & \underline{$52.3 \pm 5.0$} & \underline{$62.5$} & $\mathbf{3.25}$ \\
  \bottomrule
  \end{tabular}
  \label{tab:uea_drop_per_dataset_drop70}
\end{table*}

\begin{table*}[t]
  \centering
  \scriptsize
  \setlength{\tabcolsep}{4.5pt}
    \caption{\textbf{Per-dataset test accuracy (mean $\pm$ std over 5 seeds, in \%) for the UEA benchmark at 95\% drop, with macro-average accuracy and average dataset rank.} Lower rank is better.}
  \begin{tabular}{lcccccccc}
  \toprule
  Model & EigenWorms & Ethanol & Heartbeat & MotorImagery & SCP1 & SCP2 & Avg.\ Acc. & Avg.\ Rank \\
  \midrule
  S5 & $61.1 \pm 3.5$ & $27.3 \pm 1.3$ & $70.3 \pm 6.3$ & $\mathbf{53.7 \pm 5.7}$ & $81.9 \pm 4.4$ & $45.6 \pm 6.1$ & $56.7$ & $5.33$ \\
  LRU & $\mathbf{83.9 \pm 2.7}$ & $25.1 \pm 4.8$ & $71.6 \pm 4.6$ & $46.7 \pm 6.5$ & $81.6 \pm 4.5$ & $\mathbf{56.1 \pm 8.4}$ & $60.8$ & \underline{$4.17$} \\
  S6 & $78.8 \pm 9.7$ & $25.7 \pm 6.9$ & $71.5 \pm 4.1$ & $50.9 \pm 4.2$ & $74.9 \pm 3.7$ & $51.7 \pm 3.2$ & $58.9$ & $5.17$ \\
  Mamba & $72.2 \pm 12.2$ & $27.2 \pm 1.8$ & $71.0 \pm 8.2$ & $46.6 \pm 5.0$ & $78.0 \pm 1.8$ & $50.5 \pm 5.0$ & $57.6$ & $6.17$ \\
  NCDE & $77.2 \pm 8.3$ & $28.1 \pm 2.5$ & $68.7 \pm 4.3$ & $50.2 \pm 5.2$ & $76.9 \pm 6.6$ & $50.9 \pm 3.5$ & $58.7$ & $5.33$ \\
  Log-NCDE & $81.7 \pm 3.8$ & $\mathbf{34.9 \pm 4.7}$ & $73.5 \pm 6.8$ & $49.8 \pm 4.5$ & $76.0 \pm 4.7$ & $49.5 \pm 5.7$ & $60.9$ & \underline{$4.17$} \\
  D-SLiCE & $80.6 \pm 5.3$ & \underline{$30.6 \pm 7.1$} & $\mathbf{75.2 \pm 4.8}$ & $49.1 \pm 4.6$ & \underline{$84.5 \pm 2.6$} & \underline{$53.0 \pm 3.6$} & \underline{$62.2$} & $\mathbf{2.83}$ \\
  BD-SLiCE & \underline{$83.3 \pm 3.5$} & $30.1 \pm 9.3$ & \underline{$74.2 \pm 7.8$} & \underline{$53.0 \pm 1.7$} & $\mathbf{85.2 \pm 2.6}$ & $49.5 \pm 5.1$ & $\mathbf{62.5}$ & $\mathbf{2.83}$ \\
  \bottomrule
  \end{tabular}
  \label{tab:uea_drop_per_dataset_drop95}
\end{table*}


\end{document}